\documentclass[journal]{IEEEtai}
\usepackage{url}
\usepackage[utf8]{inputenc}
\usepackage[small]{caption}
\usepackage{graphicx}
\usepackage{booktabs}
\usepackage{algorithm}
\usepackage{cite}
\usepackage{algorithmic}
\usepackage{color}
\usepackage{multirow}
\usepackage{bm}

\usepackage{booktabs}
\usepackage{algorithm,algorithmic}

\usepackage{subfigure}
\usepackage{amsmath,amssymb, amsthm,amsfonts}
\newtheorem{Def}{Definition}

\newtheorem{theorem}{Theorem}

\ifCLASSINFOpdf
\else
\fi
\begin{document}
%
\title{V$^3$H: View Variation and View Heredity for Incomplete Multi-view Clustering}
%
%
%

\author{Xiang~Fang,
        Yuchong~Hu,~\IEEEmembership{Member,~IEEE,}
        Pan~Zhou,~\IEEEmembership{Senior Member,~IEEE,}
        and~Dapeng~Oliver~Wu,~\IEEEmembership{Fellow,~IEEE}
        
\thanks{This work is supported by National Natural Science Foundation of China (NSFC) under grant no. 61972448. (\emph{Corresponding author: Pan~Zhou}.)}
\thanks{X. Fang is with the School of Computer Science and Technology, Huazhong University of Science and Technology, Wuhan 430074, China (e-mail: xfang9508@gmail.com).}
\thanks{Y. Hu is with the School of Computer Science and Technology, Huazhong University of Science and Technology, Wuhan 430074, China (e-mail: yuchonghu@hust.edu.cn).}
\thanks{P. Zhou is with the Hubei Engineering Research Center on Big Data Security, School of Cyber Science and Engineering, Huazhong University of Science and Technology, Wuhan 430074, China (e-mail: panzhou@hust.edu.cn).}
\thanks{D. O. Wu is with the Department of Electrical and Computer Engineering, University of Florida, Gainesville, FL 32611 USA (e-mail: dpwu@ieee.org).}
\thanks{$\copyright$ 2021 IEEE. Personal use of this material is permitted.  Permission from IEEE must be obtained for all other uses, in any current or future media, including reprinting/republishing this material for advertising or promotional purposes, creating new collective works, for resale or redistribution to servers or lists, or reuse of any copyrighted component of this work in other works.}}
\maketitle

\begin{abstract}
Real data often appear in the form of multiple incomplete views. Incomplete multi-view clustering is an effective method to integrate these incomplete views.
Previous methods only learn the consistent information between different views and ignore the unique information of each view, which limits their clustering performance and generalizations. To overcome this limitation, we propose a novel \textbf{V}iew \textbf{V}ariation and \textbf{V}iew \textbf{H}eredity approach (V$^3$H). Inspired by the variation and the heredity in genetics, V$^3$H first decomposes each subspace into a variation matrix for the corresponding view and a heredity matrix for all the views to represent the unique information and the consistent information respectively. Then, by aligning different views based on their cluster indicator matrices, V$^3$H integrates the unique information from different views to improve the clustering performance. Finally, with the help of the adjustable low-rank representation based on the heredity matrix, V$^3$H recovers the underlying true data structure to reduce the influence of the large incompleteness.
%
%
More importantly, V$^3$H presents possibly the \textbf{first} work to introduce genetics to clustering algorithms for learning simultaneously the consistent information and the unique information from incomplete multi-view data. Extensive experimental results on fifteen benchmark datasets validate its superiority over other state-of-the-arts.
\end{abstract}
%
%
%

%
%
%
%
%
%
%
%
%
%

\begin{IEEEImpStatement}
Incomplete multi-view clustering is a popular technology to cluster incomplete datasets from multiple sources. The technology is becoming more significant due to the absence of the expensive requirement of labeling these datasets.
However, previous algorithms cannot fully learn the information of each view. Inspired by variation and heredity in genetics, our proposed algorithm V$^3$H fully learns the information of each view. Compared with the state-of-the-art algorithms, V$^3$H improves clustering performance by more than 20\% in representative cases. With the large improvement on multiple datasets, V$^3$H has wide potential applications including the analysis of pandemic, financial and election datasets. The DOI of our codes is 10.24433/CO.2119636.v1.
%
\end{IEEEImpStatement}


\ifCLASSOPTIONpeerreview
 \fi
%
\IEEEpeerreviewmaketitle

\section{Introduction}\label{section:intro}
\IEEEPARstart{I}{n} most real-world applications, the collected data always appear in multiple views or come from different sources~\cite{chao2018survey,liu2023exploring,wang2025taylor,fang2026towardsicml,kuai2026dynamic,wang2025point,fang2025your,zhang2025monoattack,fang2023hierarchical,liu2024towards,yang2025eood,fang2022multi,fang2026cogniVerse,lei2025exploring,fang2023you,wang2025dypolyseg,fang2025hierarchical,yan2026fit,fang2025adaptive,wang2026topadapter,cai2025imperceptible,fang2026slap,wang2026reasoning,fang2026immuno,wang2026biologically,fang2026disentangling,wang2025reducing,fang2026advancing,fang2026unveiling,wang2026from,liu2023conditional,liu2026attacking,fang2026rethinking,wang2025seeing,fang2026towards,fang2025multi,fang2024fewer,liu2024pandora,fang2024multi,fang2025turing,fang2024not,liu2023hypotheses,fang2024rethinking,liu2024unsupervised,fang2023annotations,xiong2024rethinking,fang2021unbalanced,wang2025prototype,zhang2025manipulating,fang2026align,tang2024reparameterization,fang2025adaptivetai,tang2025simplification,fang2021animc,cai2026towards,fang2020double}, which are called \emph{multi-view data}~\cite{CHAO2019278}. As an illustration, severe acute respiratory syndrome coronavirus 2 (SARS-CoV-2) is the strain of coronavirus that causes coronavirus disease 2019 (COVID-19)~\cite{REMUZZI20201225,10.1001/jama.2020.2565,ROTHAN2020102433}. When we analyze the structural proteins of SARS-CoV-2 to study vaccines, the S (spike), E (envelope), M (membrane), and N (nucleocapsid) proteins can serve as four views~\cite{WALLS2020281,10.1001/jama.2020.3786,doi:10.1056/NEJMc2001737,Yan1444}.  Also, the financial data often come from multiple sources, and each data source corresponds to a view. Multi-view clustering provides a natural way to integrate these views~\cite{bickel2004multi,chaudhuri2009multi,cai2013multi,chen2020multi}.

Recently, many multi-view clustering algorithms have been proposed~\cite{nie2020auto,han2020multi,yang2020uniform}. Most of them assume that each view is complete. But real-world data often suffer from incompleteness~\cite{zhao2016incomplete,xu2018partial,7840701}. For example, when detecting COVID-19, the blood test, the temperature measurement, and the neuroimage can be regarded as three views of the detection. But when the disease first broke out, most individuals only perform one or two tests due to the lack of detection conditions.
As such, this incompleteness may lead to the lack of columns or rows in the view matrix, which fails previous algorithms.

To cluster incomplete multi-view data, some efforts have been made.
PVC~\cite{li2014partial} aligns the same samples in different views by constructing a latent subspace.
To solve the incomplete multi-modality clustering problem, IMG~\cite{zhao2016incomplete} transforms the collected incomplete multi-view data to a complete representation in a latent space. To extend PVC to more than two views, MIC~\cite{shao2015multiple} extends MultiNMF~\cite{liu2013multi} based on weighted nonnegative matrix factorization (NMF)~\cite{lee1999learning} and $L_{2,1}$-norm regularization. To decrease the impact of a large missing rate, DAIMC~\cite{hu2018doubly} extends MIC by combining semi-nonnegative matrix factorization (semi-NMF)~\cite{ding2010convex} and $L_{2,1}$-Norm regularization regression.
To solve the multi-view co-clustering with incomplete data problem, \cite{CHAO2019278}
integrates complex patterns of incomplete multi-view data.
To explore the local structure, UEAF~\cite{wen2019unified} learns a consensus representation for all views.

However, these incomplete multi-view clustering methods still have three main drawbacks. First, these methods only consider the consistent information between all the views and ignore the unique information of each view. For example, DAIMC performs clustering by aligning the consistent information. When clustering the data with little alignment information, DAIMC cannot achieve satisfactory performance due to ignoring the unique information. The unique information from each view can help us better analyze the geometry of the corresponding view. Second, these methods are difficult to learn nonlinear information between samples. Most of these methods are based on NMF. NMF is a linear operation, which only extracts the linear structure information among samples from the data. When processing some datasets with nonlinear structures, these methods may ignore much important nonlinear information among samples, which limits the application of these methods. Third, these methods do not perform well in some datasets with relatively large missing rates.
Based on these samples presented in all views, these algorithms learn the structure information of datasets for clustering. As the missing rate increases, the number of presented samples will decrease significantly. Thus, these algorithms cannot learn enough information and will obtain poor clustering performance.
If we directly use these methods to process important data, it is difficult to learn accurate data information because of the above drawbacks. As an illustration, if we use these methods to directly analyze the data of SARS-CoV-2, these methods are difficult to obtain satisfactory results, which may lead to slow research progress on COVID-19.
Therefore, incomplete multi-view clustering still contains significant issues.

To address these issues, we propose a novel \textbf{V}iew \textbf{V}ariation and \textbf{V}iew \textbf{H}eredity approach (V$^3$H).
By introducing biological heredity and biological variation, the theory of genetics can effectively analyze the consistent trait information and the unique trait information in the biological world~\cite{dobzhansky1970genetics,mayr1970populations,hedrick2011genetics}.
Inspired by the theory, V$^3$H first learns a subspace from each view and decomposes each subspace into a heredity matrix shared by all the views and a variation matrix of the corresponding view.
The shared heredity matrix can extract the consistent information between all the views, while each variation matrix can learn the unique information of the corresponding view. Based on each variation matrix, V$^3$H constructs a graph Laplacian to obtain a corresponding cluster indicator matrix.
Then, V$^3$H aligns different views by minimizing the disagreement between each cluster indicator matrix and the consensus cluster indicator matrix. To measure the disagreement, V$^3$H introduces the linear kernel into the Laplacian for
spectral clustering.
Finally, instead of learning the low-rank representation of all the subspaces, V$^3$H designs an adjustable low-rank representation model via the $\eta$-norm of the variation matrix and the $\tau$-norm of error matrices.
V$^3$H's contributions are mainly summarized as follows:
\begin{itemize}
\item To our best knowledge, V$^3$H is a \textbf{pioneering} work to introduce genetics into the clustering algorithm, which will promote the intersection between the clustering algorithm and other disciplines. Moreover, it is also the \textbf{first} attempt to learn both the consistent information and the unique information of incomplete views simultaneously based on the subspace decomposition.
\item By minimizing the disagreement between each cluster indicator matrix and the consensus, V$^3$H can learn a satisfactory cluster indicator matrix for each view and integrate the unique information of these views, which improves its clustering performance.
    By introducing the linear kernel into the Laplacian,
    V$^3$H learns the nonlinear structure in the dataset, which guarantees its applicability in datasets with nonlinear structure.
\item Based on the adjustable low-rank representation model, V$^3$H can recover the underlying true data structure as needed, which helps us cluster the multi-view data with a relatively large missing rate.
\item Experimental results on fifteen benchmark datasets demonstrate the superiority of V$^3$H over other state-of-the-arts. Impressively, in terms of three evaluation metrics, V$^3$H improves the clustering performance by more than 20\% in representative cases.
\end{itemize} 

The rest of the paper is organized as follows.  Section~\ref{section:related} presents some related works.
Section~\ref{section:back} describes the notation and the background. Section~\ref{section:meth} first motivates V$^3$H's main idea, then proposes our V$^3$H approach, and finally solves it efficiently. Section~\ref{section:exp} evaluates V$^3$H's performance.
Section~\ref{section:con} concludes the paper.

\section{Related Works}  \label{section:related}
The most relevant work of this paper is incomplete multi-view clustering, and we present some related works in this section.
Recently, many incomplete multi-view clustering methods have been proposed~\cite{li2014partial,zhao2016incomplete,shao2015multiple,hu2018doubly,wen2019unified}. Based on the number of views clustered, we can divide these methods into the following two categories.

(i) Incomplete two-view clustering (e.g., PVC and IMG). Incomplete two-view clustering methods can only cluster incomplete data with two views.
PVC~\cite{li2014partial} learns the common and private latent spaces based on NMF~\cite{lee1999learning,cai2008sparse} and $L_{1}$-norm regularization. But PVC simply projects samples from each view into a common subspace and overlooks the global information among the two views.
To obtain better clustering performance on multi-modal visual datasets, IMG~\cite{zhao2016incomplete} extends PVC and removes the nonnegative constraint to simplify optimization. But both PVC and IMG can only solve the incomplete two-view clustering problem, which limits their application to incomplete data with more than two views.

(ii) Incomplete multi-view clustering (e.g., MIC, DAIMC and UEAF). Incomplete multi-view clustering methods can cluster incomplete data with more than two views.
As the first method for incomplete multi-view clustering, MIC~\cite{shao2015multiple} first fills the missing samples in each incomplete view with average feature values, then learns a common latent subspace based on weighted NMF and $L_{2,1}$-norm regularization.
But MIC only simply fills the missing samples with average feature values and if we cluster the data with a relatively large missing rate, this simply filling may result in a serious deviation.
To align the information of the presented samples, DAIMC~\cite{hu2018doubly} extends MIC via weighted semi-NMF~\cite{ding2010convex} and $L_{2,1}$-norm regularized regression. To obtain the robust clustering results, UEAF~\cite{wen2019unified} performs the unified common embedding aligned with incomplete views inferring framework. Both DAIMC and UEAF rely too much on alignment information. When clustering the dataset without enough alignment information, DAIMC and UEAF always obtain unsatisfactory performance because the loss of alignment information will reduce the availability of their models.

Note that besides three main drawbacks in Section~\ref{section:intro}, the previous methods have the above drawbacks. These drawbacks always result in unsatisfactory clustering performance, which limits the real-world applications of these methods.
%


\section{Notation and Background}  \label{section:back}
For convenience, we define some notations through the paper. All the matrices are written in uppercase. $[n]\overset{\text{def}}{=}\{1,2,\ldots,n\}$. For a matrix $\bm{A}$, its $ij$-th element and $i$-th column are denoted by $\bm{A}_{i,j}$ and $\bm{A}_i$ separately; its trace is denoted by $\text{Tr}(\bm{A})$; its Frobenius norm is denoted by $||\bm{A}||_F$; its $L_{2,1}$-norm is denoted by $||\bm{A}||_{2,1}$; its nuclear norm is denoted by $||\bm{A}||_{\ast}$. $|\cdot|$ is the absolute operator; $<\cdot,\cdot>$ is the inner product operator; $\bm{1}$ is a column vector with all elements as 1; $\bm{I}$ is an identity matrix. For a complete multi-view dataset, it is denoted by $\{\bm{U}^{(v)}\}_{v=1}^{n_v}\in \mathbb{R}^{n\times d_v}$, where $d_v$ is the feature dimension of the $v$-th view, $n$ is the sample number, and $n_v$ is the view number. When suffering from incompleteness, $\{\bm{U}^{(v)}\}_{v=1}^{n_v}$ can be updated to $\{\bm{X}^{(v)}\}_{v=1}^{n_v}\in \mathbb{R}^{m_v\times d_v}$, where $m_v$ is the number of presented samples in the $v$-th view. For the $v$-th view, $\bm{Z}^{(v)}$ is its subspace matrix; $\bm{N}^{(v)}$ is its variation matrix; $\bm{L}^{(v)}$ denotes its Laplacian graph; $\bm{E}^{(v)}$ is its error matrix; $\bm{F}^{(v)}$ is its cluster indicator matrix. For all the views of $\{\bm{X}^{(v)}\}_{v=1}^{n_v}$, $\bm{H}$ is the consensus cluster indicator matrix; $\bm{M}$ is the heredity matrix; $c$ is the cluster number. $\alpha$, $\beta$ and $\eta$ are nonnegative hyper-parameters.
\subsection{Incomplete Multi-view Data}  \label{subsection:incom}
The $v$-th original view matrix (including missing and presented samples) is represented as $\bm{U}^{(v)}\in \mathbb{R}^{n \times d_{v}}$, $n$ and $d_{v}$ are the number of samples and features, respectively.
By removing the missing samples, we can update the $v$-th original view matrix to a new view matrix $\bm{X}^{(v)}\in \mathbb{R}^{m_v\times d_v}$, where $m_v$ is the number of presented samples ($m_v<n$). To indicate the update,
we define an incomplete index matrix $\bm{W}^{(v)}\in \mathbb{R}^{m_v\times n}$~\cite{wen2019unified}
\begin{align}\label{W}
  \bm{W}_{i,j}^{(v)} \!=\! \left\{ \begin{array}{ll}
                     1, & \textrm{if the }i\textrm{-th sample is the }j\textrm{-th presented} \\ &\textrm{ sample in the }v\textrm{-th view;} \\
                     0, & \textrm{otherwise.}
                   \end{array}
  \right.
\end{align}
\subsection{Multi-view Subspace Clustering}  \label{subsection:subspace}
As an effective complete multi-view clustering method, multi-view subspace clustering (MVSC) integrates different views by first performing subspace clustering on each view and then unifying these subspaces to learn a cluster indicator matrix~\cite{gao2015multi}. The framework of MVSC is as follows
\begin{align}\label{zikongjian}
&\min\limits_{F}\sum_v(||\bm{U}^{(v)}-\bm{U}^{(v)}\bm{Z}^{(v)}-\bm{E}^{(v)}||_F^2+\beta||\bm{E}^{(v)}||_1\nonumber\\
&+\eta\text{Tr}(\bm{F}^T\bm{L}_Z^{(v)}\bm{F})) \nonumber\\
\mbox{s.t. }&\bm{F}^T\bm{F}=\bm{I},\bm{Z}^{(v)^T}\mathbf{1}=\mathbf{1},\bm{Z}_{i,i}^{(v)}=0,i\in[n],
\end{align}
where $\bm{F}$ is the cluster indicator matrix; for the $v$-th view, $\bm{L}_Z^{(v)}$ is its Laplacian graph.
 $\bm{Z}_{i,i}^{(v)}$=0 means that all diagonal elements of $\bm{Z}^{(v)}$ are 0.
\subsection{Biological Variation and Biological Heredity in Genetics} \label{subsection:shengwu}
In genetics, biological heredity denotes the passing on of traits from parents to their children; biological variation represents the unique trait information of their children.
By introducing biological heredity and biological variation, genetics provides some theories to analyze the consistent trait information and the unique trait information in the biological world~\cite{willham1963covariance,zhu1994analysis}.
In genetics, $\bm{P}$ denotes the observed trait information. Based on the theory for quantitative traits influenced by maternal~\cite{willham1963covariance,zhu1994analysis,visscher2008heritability}, $\bm{P}$ is partitioned as
\begin{align}\label{shengwugongshi}
  \bm{P}=\bm{B}_H+\bm{B}_V,
\end{align}
where
$\bm{B}_H$ denotes the biological heredity representation (i.e., consistent trait information) and $\bm{B}_V$ denotes the biological variation representation (i.e., unique trait information). Based on Eq.~\eqref{shengwugongshi}, genetics can explain the observed biological traits from the genetic level.
\begin{figure*}[t]
\centering
\includegraphics[width=1\textwidth]{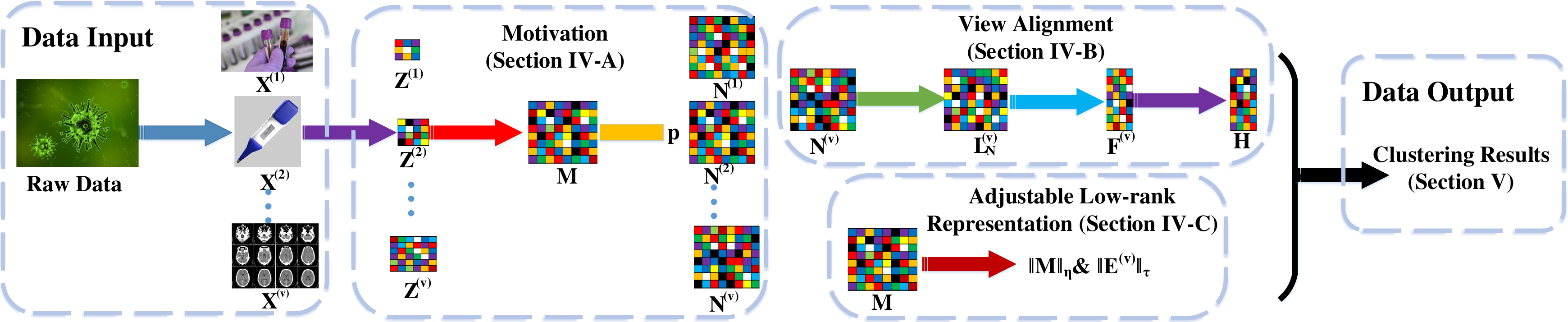}
\caption{The data flow of our proposed V$^3$H approach. Given an incomplete dataset $\{\bm{X}^{(v)}\}_{v=1}^{n_v}$,  V$^3$H learns the corresponding subspace matrices $\{\bm{Z}^{(v)}\}_{v=1}^{n_v}$. In Section~\ref{subsection:mot}, V$^3$H decomposes $\bm{Z}^{(v)}$ into a heredity matrix $\bm{M}$ and a variation matrix $\bm{N}^{(v)}$. In Section~\ref{subsection:bianyi}, by constructing a normalized graph Laplacian $\bm{L}_N^{(v)}$, V$^3$H aligns the cluster indicator matrix $\bm{F}^{(v)}$ with the consensus cluster indicator matrix $\bm{H}$. In Section~\ref{subsection:yichuan}, V$^3$H leverages $\eta$-norm and $\tau$-norm for adjustable low-rank representation. Section~\ref{section:exp} shows the clustering results.}
\label{fig:framework}
\end{figure*}
\section{Proposed V$^3$H Approach}   \label{section:meth}
By showing the  characteristics of incomplete multi-view data, we first present the motivation of our proposed V$^3$H approach. Then we model V$^3$H as the joint of the view variation and the view heredity. Finally, we design a seven-step  procedure to optimize V$^3$H.
\subsection{Motivation} \label{subsection:mot}
Real-world incomplete multi-view data have the two main characteristics~\cite{guo2019anchors}:
(i) the common samples presented in all views can be used to extract the consistent information from different views and to integrate these views; (ii) these samples existing in partial views can learn the unique information of the corresponding views.
Therefore, our motivation is to simultaneously learn consistent information and unique information from incomplete multi-view data for clustering.
Borrowing the idea of genetics in Section~\ref{subsection:shengwu}, we propose some definitions in Definition~\ref{def_qiangruo}, which relies on the following assumptions: (i) for an incomplete multi-view dataset $\{\bm{X}^{(v)}\}_{v=1}^{n_v}\in \mathbb{R}^{m_v\times d_v}$, each of its subspaces is the perturbation of a consensus subspace; (ii) each subspace can represent the data structure of the corresponding view.
\begin{Def}\label{def_qiangruo}
(Parent Subspace, Child Subspace, View Heredity, View Variation). For the $\{v\}_{v=1}^{n_v}$-th view matrix $\bm{X}^{(v)}\in\mathbb{R}^{m_v\times d_v}$, $\bm{Z}^{(v)}\in\mathbb{R}^{m_v\times m_v}$ denotes its subspace, which is calculated by Eq.~\eqref{zikongjian}. Assume that all the subspaces originate from a consensus subspace $\bm{Z}^{\ast}$. Thus, $\bm{Z}^{\ast}$ is defined as the parent subspace and $\bm{Z}^{(v)}$ is defined as the child subspace. $\bm{M}$ is defined as the heredity matrix, which represents the consistent information shared by all the views. $\bm{N}^{(v)}$ is defined as the variation matrix of the $v$-th view, which represents the unique information of the $v$-th view.
The view heredity is the phenomenon that the consistent information exists in both parent subspace and child subspace.
The view variation is the phenomenon that the unique information only exists in the corresponding child subspace.
In general, different subspaces $\{\bm{Z}^{(v)}\}_{v=1}^{n_v}$ have different variation matrices $\{\bm{Z}^{(v)}\}_{v=1}^{n_v}$, but share the same heredity matrix $\bm{M}$.
\end{Def}
Ideally, we want to obtain the optimal representation by learning the parent subspace $\bm{Z}^{\ast}$ because $\bm{Z}^{\ast}$ contains most of the available information on the data. However, it is difficult to learn an available $\bm{Z}^{\ast}$ due to missing samples and noises.

Therefore, we adopt a clever method to avoid learning $\bm{Z}^{\ast}$ directly. Instead, we can learn both the heredity matrix $\bm{M}$ and the variation matrix $\bm{N}^{(v)}$ as an alternative to learning $\bm{Z}^{\ast}$. For $\bm{Z}^{(v)}$, it can be decomposed into a heredity matrix $\bm{M}$ and a variation matrix $\bm{N}^{(v)}$. We assume that for a specific incomplete multi-view dataset, when the missing rate changes, the dimensions of its heredity matrix and variation matrices do not change. This assumption can guarantee that we can integrate the subspaces with different dimensions.

In Eq.~\eqref{shengwugongshi}, the genetics analyzes the influence of biological heredity and biological variation on biological traits.
Inspired by this, to integrate the heredity matrix and variation matrices, we design a subspace decomposition model as follows
\begin{align}\label{Zfenjie}
  \bm{W}^{(v)^T}\bm{Z}^{(v)}\bm{W}^{(v)}=\bm{M}-p\bm{N}^{(v)},
\end{align}
where $p$ is adjustable as needed; $\bm{W}^{(v)^T}\bm{Z}^{(v)}\bm{W}^{(v)}$ is composed of $\bm{M}\in\mathbb{R}^{n\times n}$ and $\bm{N}^{(v)}\in\mathbb{R}^{n\times n}$.
For Eq.~\eqref{shengwugongshi} and Eq.~\eqref{Zfenjie},
$\bm{W}^{(v)^T}\bm{Z}^{(v)}\bm{W}^{(v)}$ corresponds to $\bm{P}$; $\bm{M}$ corresponds to $\bm{B}_H$; $-p\bm{N}^{(v)}$ corresponds to $\bm{B}_V$.
Compared with $\bm{Z}^{(v)}\in\mathbb{R}^{m_v\times m_v}$, both $\bm{M}$ and $\bm{N}^{(v)}$ have larger matrix sizes ($n>m_v$). Since each column of $\bm{Z}^{(v)}$ corresponds to a sample, $\bm{Z}^{(v)}$ only contains the information of $m_v$ presented samples. Both $\bm{M}$ and $\bm{N}^{(v)}$ contain the information of $n$ samples (i.e., all the samples including the missing samples and the presented samples). Eq.~\eqref{Zfenjie} can learn three kinds of information: the consistent information between views (learned by $\bm{M}$), the unique information (learned by $\bm{N}^{(v)}$), and the relationship information between samples (learned by $\bm{W}^{(v)^T}\bm{Z}^{(v)}\bm{W}^{(v)}$). Since these three kinds of information are exactly expected in the clustering process, solving Eq.~\eqref{Zfenjie} can become a feasible alternative to learning $\bm{Z}^{\ast}$.
Note that we have three variables ($\bm{Z}^{(v)}$, $\bm{M}$, and $\bm{N}^{(v)}$), but we only have one equation (Eq.~\eqref{Zfenjie}). Therefore, in the next section, we add some constraints on $\bm{M}$ and $\bm{N}^{(v)}$ to solve Eq.~\eqref{Zfenjie}.
\subsection{View Variation for View Alignment} \label{subsection:bianyi}
Inspired by the phenomenon that the expression traits of parents and children are similar~\cite{cavalli1999genetics,feuk2006structural}, we attempt to align the expression traits of parent subspace and child subspace\footnote{In Section~\ref{subsection:mot}, we can note that one view matrix $\bm{X}^{(v)}$ and its subspace matrix $\bm{Z}^{(v)}$ are in one-to-one correspondence. For ease of description, we still refer to this subspace-based alignment as view alignment.}. We treat the cluster indicator matrices as the expression traits of the corresponding views because we can obtain better clustering results after aligning these cluster indicator matrices.
To formulate the alignment, we design the following view alignment model
\begin{align}\label{Npujulei}
\min\limits_{\bm{F}^{(v)},\bm{H}}\sum_v(\text{Tr}(\bm{F}^{(v)^T}\bm{L}_N^{(v)}\bm{F}^{(v)})+\gamma \text{Dis}(\bm{F}^{(v)},\bm{H})),
\end{align}
where $\gamma$ is a nonnegative hyper-parameter, $\bm{F}^{(v)}\in\mathbb{R}^{n\times c}$ is the cluster indicator matrix of the $v$-th view, and $\bm{H}\in\mathbb{R}^{n\times c}$ is the consensus cluster indicator matrix of all the views.  Tr($\bm{F}^{(v)^T}\bm{L}_N^{(v)}\bm{F}^{(v)}$) is used to learn the cluster indicator matrix of each view; Dis($\bm{F}^{(v)},\bm{H}$) means the disagreement between $\bm{F}^{(v)}$ and $\bm{H}$; $\bm{L}_N^{(v)}\in\mathbb{R}^{n\times n}$ is the normalized graph Laplacian, where $\bm{L}_N^{(v)}=\bm{G}^{(v)}-\bm{S}^{(v)}$, $\bm{S}^{(v)}=(|\bm{N}^{(v)}|+|\bm{N}^{(v)}|^T)/2$, and $\bm{G}_{i,i}^{(v)}=\sum_j\bm{S}_{i,j}^{(v)}$ ($i\in [n]$, $j\in [n]$).
For the measure of Dis($\bm{F}^{(v)},\bm{H}$), a popular method is using the linear kernel function, which is simple and widely used.
However, the linear kernel function can only learn linear structural information in most cases, and it is difficult to learn nonlinear information.
%
%
%


\begin{figure*}[t]
\centering
\subfigure[$\eta$-norm]{\label{fig:fanshuzuixin}
\includegraphics[width=0.47\textwidth]{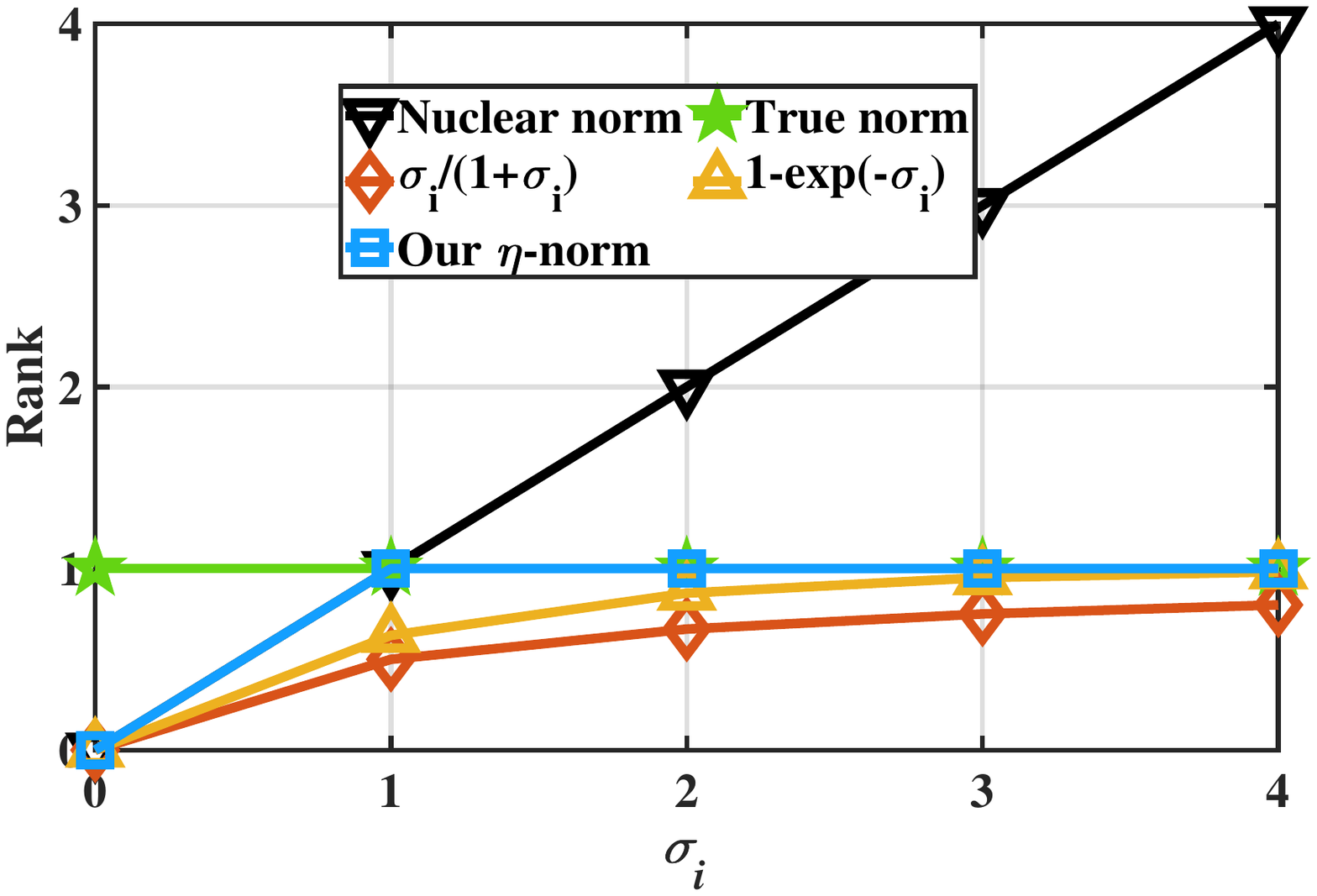}}
\subfigure[$\tau$-norm with different $\tau$]{\label{fig:taofanshu}
\includegraphics[width=0.47\textwidth]{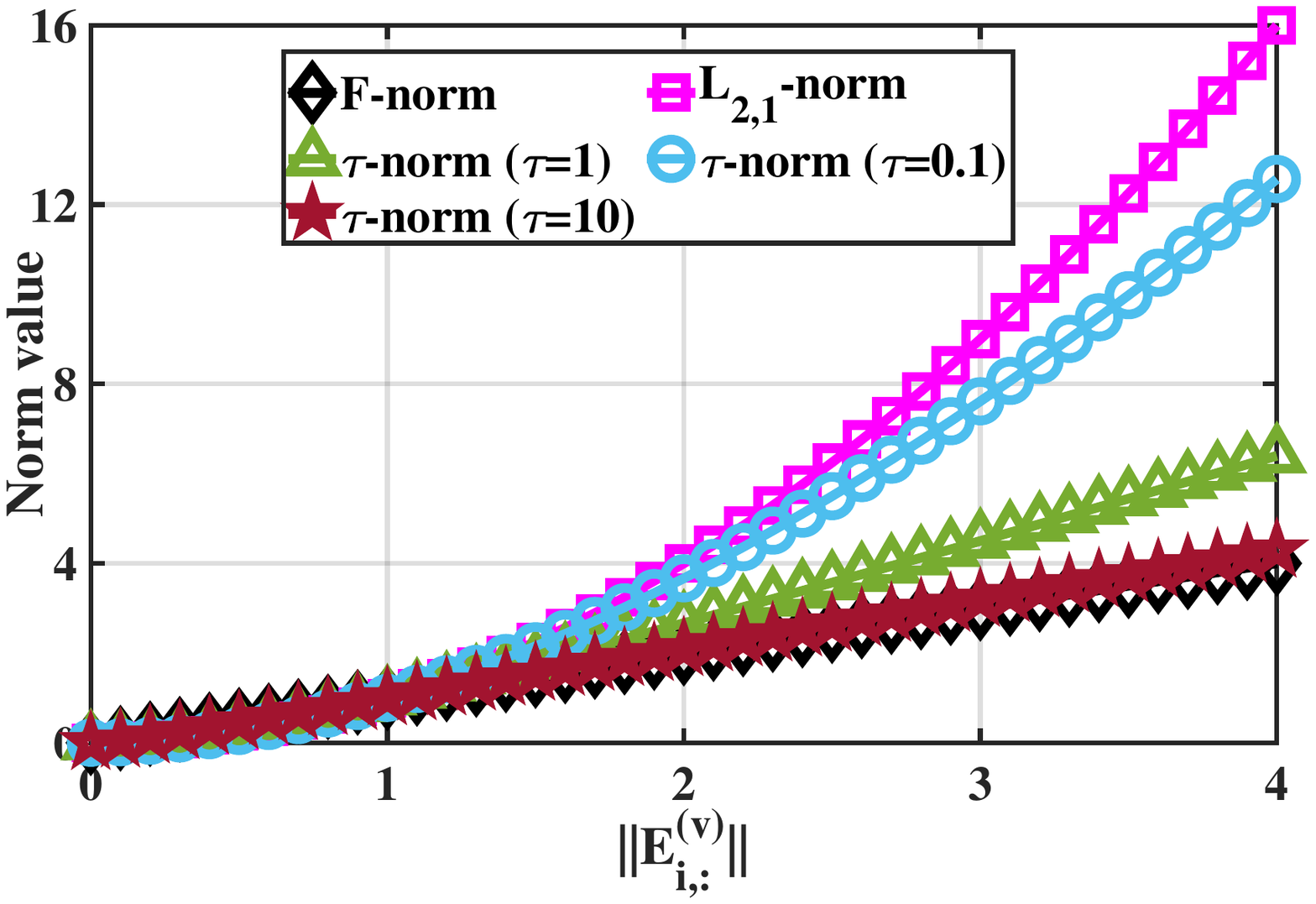}}
\caption{The performance of $\eta$-norm and $\tau$-norm.}
\label{fig:fanshu}
\end{figure*}
An ideal method is to design a linear kernel function that can learn nonlinear structures. Fortunately, the linear kernel used in the Laplacian for spectral clustering can learn the nonlinear structure of the data~\cite{kumar2011co}. Thus,
we choose the linear kernel function to measure the disagreement Dis($\bm{F}^{(v)},\bm{H}$).
Therefore, Dis($\bm{F}^{(v)},\bm{H}$) is defined as follows
\begin{align}\label{juli}
\text{Dis}(\bm{F}^{(v)},\bm{H})&=||\bm{F}^{(v)}\bm{F}^{(v)^T}-\bm{H}\bm{H}^T||_F^2,
\end{align}
where $\bm{F}^{(v)}\bm{F}^{(v)^T}$ is the linear kernel of $\bm{F}^{(v)}$ and $\bm{H}\bm{H}^T$ is the linear kernel of $\bm{H}$.
Therefore, we rewrite Eq.~\eqref{Npujulei} as
\begin{align}\label{Npujuleixin}
\min\limits_{\bm{F}^{(v)},\bm{H}}\sum_v(\text{Tr}(\bm{F}^{(v)^T}\bm{L}_N^{(v)}\bm{F}^{(v)})+\gamma||\bm{F}^{(v)}\bm{F}^{(v)^T}-\bm{H}\bm{H}^T||_F^2).
\end{align}
Combining Eq.~\eqref{zikongjian}, ~\eqref{Zfenjie}, and ~\eqref{Npujuleixin}, we can obtain
\begin{align}\label{xinzikongjian}
&\min\limits_{\bm{F}^{(v)}}\sum_v(||\bm{X}^{(v)}-\bm{X}^{(v)}\bm{Z}^{(v)}-\bm{E}^{(v)}||_F^2+\beta||\bm{E}^{(v)}||_{1}\nonumber\\
&+\alpha\text{Tr}(\bm{F}^{(v)^T}\bm{L}_N^{(v)}\bm{F}^{(v)})+\gamma||\bm{F}^{(v)}\bm{F}^{(v)^T}-\bm{H}\bm{H}^T||_F^2) \nonumber\\
\mbox{s.t. }&\bm{W}^{(v)^T}\bm{Z}^{(v)}\bm{W}^{(v)}=\bm{M}-p\bm{N}^{(v)},\bm{F}^{(v)^T}\bm{F}^{(v)}=\bm{I},\nonumber\\
&\bm{H}^T\bm{H}=\bm{I},\bm{N}^{(v)}\mathbf{1}=\mathbf{1},\bm{N}_{i,i}^{(v)}=0,i\in[n],
\end{align}
where the constraint $\bm{N}^{(v)}\mathbf{1}=\mathbf{1}$ treats all samples equally, which can learn the unique information of the $v$-th view; the constraint $\bm{N}_{i,i}^{(v)}=0$ can ensure that each sample can only be represented as the combination of other samples~\cite{vidal2011subspace}.
\subsection{View Heredity for Adjustable Low-rank Representation} \label{subsection:yichuan}
Most real-world multi-view data often have the low-rank subspace representations, which can be used to recover the underlying true data structure~\cite{liu2012robust,candes2011robust,He2007ARS}. Thus, the reasonable low-rank representation can improve the performance of subspace-based clustering.
Outlier Pursuit \cite{xu2010robust} is a popular technology to obtain the proper low-rank representation, which is formulated as
\begin{align}\label{Mhefanshu}
\min\limits_{\bm{M},\bm{E}^{(v)}}||\bm{M}||_{\ast}+\beta\sum_v||\bm{E}^{(v)}||_{2,1},
\end{align}
where $||\bm{M}||_{\ast}$ is used to approximate the rank of $\bm{M}$; $||\bm{E}^{(v)}||_{2,1}$ can learn the low-rank representation of $\bm{E}^{(v)}$.

In fact, on the one hand, the rank approximation of $\bm{M}$ by the nuclear function will lead to a large deviation, which may result in unusable clustering results.
On the other hand, $L_{2,1}$-norm is indifferentiable at the point of zero, which renders the derivate of $L_{2}$-norm at sparse rows senseless.
Besides, two cases fail Eq.~\eqref{Mhefanshu}: (i) If a dataset suffers from incompleteness or noises, its heredity matrix $\bm{M}$ is often unavailable. When the unavailable heredity matrix is used directly in Eq.~\eqref{Mhefanshu}, we will have difficulty obtaining the satisfactory low-rank representation because calculating the rank of the unavailable heredity matrix will produce a large deviation.
(ii) When handling different clustering tasks, we often need different $\bm{E}^{(v)}$ with different sparsity~\cite{5128907}. The adjustable sparsity of $\bm{E}^{(v)}$ is necessary.
Therefore, Eq.~\eqref{Mhefanshu} also has two drawbacks: unavailable heredity matrix and non-adjustable sparsity.

The nonconvex relaxation of matrix rank is a popular technique.
Evoked by \cite{doi:10.1080,doi:10.1198,392335}, to learn available $\bm{M}$,
we propose $\eta$-norm defined by
\begin{align}\label{sigema}
||\bm{M}||_{\eta}=\sum_i\frac{(\eta+w_i)\sigma_i(\bm{M})}{\eta+w_i\sigma_i(\bm{M})},
\end{align}
where $\eta$ is adjustable as needed ($\eta>0$); $w$ is the weight vector ($w_i>0$); $\sigma_i(\bm{M})$
is the $i$-th singular value of $\bm{M}$.

Note that different from the common norm (e.g., $L_{2,1}$-norm and $L_F$-norm, etc.), our proposed $\eta$-norm is not a real norm.
$\eta$-norm has the following characteristics:

1) $\eta$-norm is unitarily invariant, and $||\bm{M}||_{\eta}=||\bm{U}\bm{M}\bm{V}||_{w}$ for any orthonormal $\bm{U}\in \mathbb{R}^{m\times m}$ and $\bm{V}\in \mathbb{R}^{n\times n}$;

2) when $\eta\rightarrow\infty$, we have $||\bm{M}||_{\eta}\rightarrow||\bm{M}||_{w,\ast}$, where $||\cdot||_{w,\ast}$ is the weighted nuclear norm \cite{gu2014weighted};

3) when $\eta\rightarrow 0$, we have $||\bm{M}||_{\eta}\rightarrow$ rank($\bm{M}$).

To show the advantage of our proposed $\eta$-norm, we compare $\eta$-norm with several rank relaxation approaches (i.e., $\sigma_i/(1+\sigma_i)$ and $1-exp(-\sigma_i)$ in Fig.~\ref{fig:fanshuzuixin}). As shown in Fig.~\ref{fig:fanshuzuixin}, $\eta$-norm ($\eta$=0.001 in this figure) is closer to the true rank than other approaches. Thus, by learning the satisfactory low-rank representation, $\eta$-norm can ensure the availability of the heredity matrix.
Similar to Eq.~\eqref{sigema}, to learn an adjustable sparse representation, we adopt the $\tau$-norm of matrix $\bm{E}^{(v)}$ defined by
\begin{align}\label{tau}
||\bm{E}^{(v)}||_{\tau}=\sum_i\frac{(1+\tau)||\bm{E}_{i,:}^{(v)}||^2}{\tau+||\bm{E}_{i,:}^{(v)}||},
\end{align}
where $\tau$ is adjustable for different tasks. Based on matrix $\bm{E}^{(v)}$, we design a diagonal matrix $\bm{D}_E^{(v)}$ defined by
\begin{align}\label{ddingyi}
\bm{D}_{E_{i,i}}^{(v)}=\frac{(1+\tau)(||\bm{E}_{i}^{(v)}||+2\tau)}{(||\bm{E}_{i}^{(v)}||+\tau)^2}.
\end{align}
\begin{theorem}\label{sigemadingli}
For any matrix $\bm{E}^{(v)}$, we have
\begin{align}\label{piandingli}
\frac{\partial ||\bm{E}^{(v)}||_{\tau}}{\partial\bm{E}^{(v)}}=\bm{D}_E^{(v)}\bm{E}^{(v)}.
\end{align}
\end{theorem}
\begin{proof}\label{zhengpiandingli}
\begin{align}\label{piandinglizheng}
\frac{\partial ||\bm{E}^{(v)}||_{\tau}}{\partial \bm{E}_{i,:}^{(v)}}&=\frac{\partial((1+\tau)||\bm{E}_{i,:}^{(v)}||_2^2/(1+\tau||\bm{E}_{i,:}^{(v)}||_2))}{\partial \bm{E}_{i,:}^{(v)}}\nonumber\\
&=\frac{(1+\tau)(||\bm{E}_{i}^{(v)}||+2\tau)}{(||\bm{E}_{i}^{(v)}||+\tau)^2}\bm{E}_{i,:}^{(v)}\nonumber\\
\bm{D}_{i,i}^{(v)}\bm{E}_{i,:}^{(v)}&=\frac{(1+\tau)(||\bm{E}_{i}^{(v)}||+2\tau)}{(||\bm{E}_{i}^{(v)}||+\tau)^2}\bm{E}_{i,:}^{(v)}.
\end{align}
Obviously, $\partial ||\bm{E}^{(v)}||_{\tau}/{\partial \bm{E}_{i,:}^{(v)}}=\bm{D}_{i,i}^{(v)}\bm{E}_{i,:}^{(v)}$ in all the cases.
\end{proof}
Similar to $\eta$-norm, $\tau$-norm has the following characteristics:

1) $\tau$-norm is nonnegative and global differentiable;

%

2) when $\tau\rightarrow\infty$, we have $||\bm{E}^{(v)}||_{\tau}\rightarrow||\bm{E}^{(v)}||_F^2$ and $\bm{D}_E^{(v)}\rightarrow\bm{I}$;

3) when $\tau\rightarrow 0$, we have $||\bm{E}^{(v)}||_{\tau}\rightarrow||\bm{E}^{(v)}||_{2,1}$ and $\bm{D}_{E_{i,i}}^{(v)}\rightarrow 1/||\bm{E}_{i,:}^{(v)}||_2$.

To verify the adjustability of $\tau$-norm, we compare $\tau$-norm with $L_{2,1}$-norm and F-norm in Fig.~\ref{fig:taofanshu}. When $\tau$ is relatively small ($\tau=0.1$), $\tau$-norm is near to F-norm.
As $\tau$ decreases, $||\bm{E}^{(v)}||_{\tau}$ is closer to $||\bm{E}^{(v)}||_{2,1}$, and $\bm{E}^{(v)}$ becomes more sparse. Since we can choose different $\tau$ to adjust the sparsity, $\tau$-norm has wider applications than $L_{2,1}$-norm and F-norm.

Therefore, considering both the available low-rank subspaces (Eq.~\eqref{sigema}) and the adjustable sparsity (Eq.~\eqref{tau}),
we can obtain the adjustable low-rank representation as follows
\begin{align}\label{Mhefanshuxin}
\min\limits_{\bm{M},\bm{E}^{(v)}}||\bm{M}||_{\eta}+\beta\sum_v||\bm{E}^{(v)}||_{\tau}.
\end{align}
\subsection{Objective Function}
Combining the view alignment (Eq.~\eqref{xinzikongjian}) and the adjustable low-rank representation (Eq.~\eqref{Mhefanshuxin}), we have
\begin{align}\label{zikongjiangai2}
&\min\limits_{\bm{M},\bm{N}^{(v)},\bm{Z}^{(v)},\bm{E}^{(v)},\bm{F}^{(v)},\bm{F}^{\ast}}||\bm{M}||_{\eta}+\sum_v(\beta||\bm{E}^{(v)}||_{\tau}\nonumber\\
&+\alpha\text{Tr}(\bm{F}^{(v)^T}\bm{L}_N^{(v)}\bm{F}^{(v)})+\gamma ||\bm{F}^{(v)}\bm{F}^{(v)^T}-\bm{H}\bm{H}^T||_F^2)\\
&\mbox{s.t. }\bm{F}^{(v)^T}\bm{F}^{(v)}=\bm{I},\bm{H}^T\bm{H}=\bm{I},\bm{N}^{(v)}\mathbf{1}=\mathbf{1},\bm{N}_{i,i}^{(v)}=0,\nonumber\\
&\bm{W}^{(v)^T}\bm{Z}^{(v)}\bm{W}^{(v)}=\bm{M}-p\bm{N}^{(v)},\bm{X}^{(v)}=\bm{X}^{(v)}\bm{Z}^{(v)}-\bm{E}^{(v)}.\nonumber
\end{align}
Eq.~\eqref{zikongjiangai2} is a nonconvex function, which is often difficult to optimize directly. In the next section, we will design an iteration procedure to optimize it.
\subsection{Optimization}\label{subsection:optimization}
To optimize Eq.~\eqref{zikongjiangai2}, we design the following augmented Lagrangian function
\begin{align}\label{lagelangrichu}
&J=||\bm{M}||_{\eta}+\sum_v(\alpha\text{Tr}(\bm{F}^{(v)^T}\bm{L}_N^{(v)}\bm{F}^{(v)})+\beta||\bm{E}^{(v)}||_{\tau}\nonumber\\
&+\gamma ||\bm{F}^{(v)}\bm{F}^{(v)^T}-\bm{H}\bm{H}^{T}||_F^2+\frac{\omega}{2}(||\bm{W}^{(v)^T}\bm{Z}^{(v)}\bm{W}^{(v)}\nonumber\\
&-\bm{M}+p\bm{N}^{(v)}+\frac{\bm{C}_2^{(v)}}{\omega}||_F^2+||\bm{X}^{(v)}-\bm{X}^{(v)}\bm{Z}^{(v)}-\bm{E}^{(v)}\nonumber\\
&+\frac{\bm{C}_1^{(v)}}{\omega}||_F^2-\zeta^{(v)^T}(\bm{N}^{(v)}\bm{1}-\bm{1}))),
\end{align}
where matrices $\bm{C}_1^{(v)}$, $\bm{C}_2^{(v)}$ and vector $\zeta^{(v)}$ are Lagrange multipliers, $\omega$ is a nonnegative penalty parameter.
Eq.~\eqref{lagelangrichu} is not convex for all variables simultaneously, and it is difficult to solve Eq.~\eqref{lagelangrichu} in one step. Thus, we design the following seven-step procedure to update each variable iteratively~\cite{kang2015robust}.
\begin{algorithm}[t]
    \caption{V$^3$H}
    \label{simc_alg}
	\begin{algorithmic}
	 \REQUIRE{$\{\bm{X}^{(v)}\}_{v=1}^{n_v}$, $\{\bm{W}^{(v)}\}_{v=1}^{n_v}$, $\alpha,\beta,\gamma$, $\lambda_{\max}$, and $c$.}
    \STATE{Initialize $\bm{E}^{(v)}=\bm{Z}^{(v)}=\bm{M}=\bm{N}^{(v)}=\bm{0}$, $\bm{C}_1^{(v)}=\bm{C}_2^{(v)}=\bm{0}$, $\bm{F}^{(v)}$, $\bm{F}^{\ast}$, and $\omega$.}
	\REPEAT
	\FOR{$v=1$ to $n_v$}
	\STATE{Update $\bm{\bm{Z}}^{(v)}$ by Eq.~\eqref{zjie};}
	\STATE{Update $\bm{N}^{(v)}$ and $\zeta$ based on Eq.~\eqref{nzuizhongjie}, Eq.~\eqref{zeta}, and the constraint $\bm{N}_{i,:}^{(v)}\bm{1}-1=0$;}
	\STATE{Update $\bm{E}^{(v)}$ by Eq.~\eqref{ejie};}
	\STATE{Update $\bm{F}^{(v)}$ by Eq.~\eqref{jf};}
    \STATE{Update $\bm{C}_1^{(v)},\bm{C}_2^{(v)}$, and $\omega$ by Eq.~\eqref{jc};}
	\ENDFOR
	\STATE{Update $\bm{M}$ by Eq. \eqref{mjie} and DC programming;}
	\STATE{Update $\bm{H}$ by Eq.~\eqref{jfxing};}
	\UNTIL{converges}
    \ENSURE{$\bm{F}^{(v)}$, $\bm{H}$, $\bm{M}$, $\bm{N}^{(v)}$ and clustering results.}
	\end{algorithmic}
\end{algorithm}

\noindent\textbf{Step 1.} Update $\bm{M}$. Fixing the other variables, the problem to update $\bm{M}$ is degraded to solve the following problem
\begin{align}\label{mjie}
&\bm{M}=\arg\min\limits_{\bm{M}}\frac{\omega}{2}||\bm{W}^{(v)^T}\bm{Z}^{(v)}\bm{W}^{(v)}-\bm{M}+p\bm{N}^{(v)}\nonumber\\
&+\frac{\bm{C}_2^{(v)}}{\omega}||_F^2+||\bm{M}||_{\eta}.
\end{align}
To solve Eq.~\eqref{mjie}, we first develop the following theorem.
\begin{theorem}\label{mfanshudingli}
We first set $\bm{A}=\bm{W}^{(v)^T}\bm{Z}^{(v)}\bm{W}^{(v)}+p\bm{N}^{(v)}+\bm{C}_2^{(v)}/{\omega}$. The SVD operation of $\bm{A}$ is $\bm{A}=\bm{U}\bm{\Sigma}_A\bm{V}^T$, where $\bm{\Sigma}_A=diag(\sigma_A)$. Set $H(\bm{M})=h\circ\sigma_M$
be a unitarily invariant function, where $h(\sigma)=\Sigma_i(\eta+w_i)\sigma_i(\bm{M})/(\eta+w_i\sigma_i(\bm{M}))$ and $\omega>0$. Based on Eq.~\eqref{mjie}, we have
\begin{align}\label{hm}
\min\limits_{\bm{M}}H(\bm{M})+\frac{\omega}{2}||\bm{M}-\bm{A}||_F^2.
\end{align}
Then an optimal solution to Eq.~\eqref{hm} is $\bm{M}^{\ast}=\bm{U}\bm{\Sigma}_M^{\ast}\bm{V}^T$, where $\bm{\Sigma}_M^{\ast}=diag(\sigma^{\ast})$ and $\sigma^{\ast}=prox_{h,\omega}(\sigma_A)$. $prox_{h,\omega}(\sigma_A)$ is the \emph{Moreau-Yosida operator}, which is defined as
\begin{align}\label{prox}
prox_{h,\omega}(\sigma_A)=\arg\min h(\sigma)+\frac{\omega}{2}||\sigma-\sigma_A||_2^2.
\end{align}
\end{theorem}
\begin{proof}\label{mfanshuzhengming}
Since $\bm{A}=\bm{U}\bm{\Sigma}_A\bm{V}^T$, we have $\bm{\Sigma}_A=\bm{U}^T\bm{A}\bm{V}$. Denoting $\bm{Q}=\bm{U}^T\bm{M}\bm{V}$ which has the same singular values as $\bm{M}$, we have
\begin{subequations}
\begin{align}
&H(\bm{M})+\frac{\omega}{2}||\bm{M}-\bm{A}||_F^2\label{Z0}\\
=&H(\bm{Q})+\frac{\omega}{2}||\bm{Q}-\bm{\Sigma}_A||_F^2\label{ZA}\\
\geq&H(\bm{\Sigma}_Q)+\frac{\omega}{2}||\bm{\Sigma}_Q-\bm{\Sigma}_A||_F^2\label{ZB}\\
=&H(\bm{\Sigma}_M)+\frac{\omega}{2}||\bm{\Sigma}_M-\bm{\Sigma}_A||_F^2 \label{ZC}\\
=&h(\sigma)+\frac{\omega}{2}||\sigma-\sigma_A||_2^2\label{ZD}\\
\geq&h(\sigma^{\ast})+\frac{\omega}{2}||\sigma^{\ast}-\sigma_A||_2^2.\nonumber
\end{align}
\end{subequations}
Eq.~\eqref{ZA} holds because the Frobenius norm is unitarily invariant, Eq.~\eqref{ZB} holds due to the Hoffman-Wielandt inequality, and Eq.~\eqref{ZC} holds based on $\bm{\Sigma}_Q=\bm{\Sigma}_M$. Therefore, Eq.~\eqref{ZC} is a lower bound
of Eq.~\eqref{Z0}. Due to $\bm{\Sigma}_M=\bm{\Sigma}_Q=\bm{Q}=\bm{U}^T\bm{M}\bm{V}$, the singular value decomposition (SVD) of $\bm{M}$ is $\bm{M}=\bm{U}^T\bm{\Sigma}_M\bm{V}$. By minimizing Eq.~\eqref{ZD}, we learn $\sigma^{\ast}$. Hence $\bm{M}^{\ast}=\bm{U}diag(\sigma^{\ast})\bm{V}^T$, which is the optimal solution of Eq.~\eqref{hm}. Thus, Theorem~\ref{mfanshudingli} is proved.
\end{proof}
Note that Eq.~\eqref{prox} is a combination of concave and convex functions, which motivates us to leverage the difference of convex (DC) programming algorithm \cite{tao1997convex}. The algorithm decomposes a nonconvex function as the difference of two convex functions and iteratively optimizes it by linearizing the concave term at each iteration. For the $i$-th inner iteration,
$g_i=\partial h(\sigma^{i})$ denotes the gradient of $h(\cdot)$ at $\sigma^{i}$, where $\sigma^{i}$ is the updated value of $\sigma(\bm{M})$ after the $i$-th inner iteration. $\bm{U}diag(\sigma_A)\bm{V}^T$ is the SVD of $\bm{W}^{(v)^T}\bm{Z}^{(v)}\bm{W}^{(v)}+p\bm{N}^{(v)}+\bm{C}_2^{(v)}/{\omega}$. For the $(i+1)$-th inner iteration, we have
\begin{align}\label{sigmakding}
\sigma^{i+1}=\arg\min <g_i,\sigma^i>+\frac{\omega}{2}||\sigma^i-\sigma_A||_2^2,
\end{align}
which admits the following closed-form solution
\begin{align}\label{sigmakjieguo}
\sigma^{i+1}=(\sigma_A-\frac{g_i}{\omega^i}).
\end{align}
After several iterations, it at least converges to a locally optimal point $\sigma^{\ast}$. Then $\bm{M}=\bm{U}diag(\sigma^{\ast})\bm{V}^T$.

\noindent\textbf{Step 2.} Update $\bm{Z}^{(v)}$. Fixing the other variables,
the problem to update $\bm{Z}^{(v)}$ is degraded to minimize
\begin{align}\label{jz}
&J(\bm{Z}^{(v)})=||\bm{W}^{(v)^T}\bm{Z}\bm{W}^{(v)}-\bm{M}+p\bm{N}^{(v)}+\frac{\bm{C}_2^{(v)}}{\omega}||_F^2\nonumber\\
&+||\bm{X}^{(v)}-\bm{X}^{(v)}\bm{Z}^{(v)}-\bm{E}^{(v)}+\frac{\bm{C}_1^{(v)}}{\omega}||_F^2.
\end{align}
Setting the derivative $J(\bm{Z}^{(v)})$ w.r.t $\bm{Z}^{(v)}$ to $\bm{0}$, we have
\begin{align}\label{jzpian}
&2\bm{W}^{(v)}(\bm{W}^{(v)^T}\bm{Z}\bm{W}^{(v)}-\bm{M}+p\bm{N}^{(v)}+\frac{\bm{C}_2^{(v)}}{\omega})\bm{W}^{(v)^T}\nonumber\\
&+2\bm{X}^{(v)^T}(\bm{X}^{(v)}-\bm{X}^{(v)}\bm{Z}^{(v)}-\bm{E}^{(v)}+\frac{\bm{C}_1^{(v)}}{\omega})=\bm{0}.
\end{align}
Based on the definition of $\bm{W}^{(v)}$, we can find $\bm{W}^{(v)}\bm{W}^{(v)^T}=\bm{I}$. By solving Eq.~\eqref{jzpian}, we can update $\bm{Z}^{(v)}$ by
\begin{align}\label{zjie}
&\bm{Z}^{(v)}=(\bm{I}+\bm{X}^{(v)^T}\bm{X}^{(v)})^{-1}(\bm{X}^{(v)^T}(\bm{X}^{(v)}-\bm{E}^{(v)}\\
&+\frac{\bm{C}_1^{(v)}}{\omega})+\bm{W}^{(v)}(\bm{M}-p\bm{N}^{(v)}-\frac{\bm{C}_2^{(v)}}{\omega})\bm{W}^{(v)^T}).\nonumber
\end{align}
\noindent\textbf{Step 3.} Update $\bm{N}^{(v)}$ and $\zeta$. Fixing the other variables, the problem to update $\bm{N}^{(v)}$ is degraded to minimize
\begin{align}\label{jn}
&J(\bm{N}^{(v)})=\frac{\omega}{2}||\bm{W}^{(v)^T}\bm{Z}^{(v)}\bm{W}^{(v)}-\bm{M}+p\bm{N}^{(v)}+\frac{\bm{C}_2^{(v)}}{\omega}||_F^2\nonumber\\
&+\alpha\text{Tr}(\bm{F}^{(v)^T}\bm{L}_N^{(v)}\bm{F}^{(v)})-\zeta^{(v)^T}(\bm{N}^{(v)}\bm{1}-\bm{1}).
\end{align}
We define $\bm{K}^{(v)}=\bm{W}^{(v)^T}\bm{Z}^{(v)}\bm{W}^{(v)}-\bm{M}+\bm{C}_2^{(v)}/{\omega}$, and Eq.~\eqref{jn} can be equivalent to
\begin{align}\label{jnjianhua}
&J(\bm{N}^{(v)})=\frac{\omega p}{2}||\bm{N}^{(v)}-\frac{1}{p}\bm{K}^{(v)}||_F^2+\frac{\alpha}{2}\sum_{i,j}^m||\bm{F}_{i,:}^{(v)}\nonumber\\ &-\bm{F}_{j,:}^{(v)}||_2^2\bm{N}^{(v)}-\zeta^{(v)^T}(\bm{N}^{(v)}\bm{1}-\bm{1}).
\end{align}
Note that Eq.~\eqref{jnjianhua} is independent to each row. Defining $\bm{T}_{i,j}^{(v)}$=$||\bm{F}_{i,:}^{(v)}-\bm{F}_{j,:}^{(v)}||_2^2$, we transform minimizing Eq.~\eqref{jnjianhua} into
\begin{align}\label{njie}
\min\limits_{\bm{N}_{i,:}^{(v)}\geq 0,\bm{N}_{i,i}^{(v)}=0}&\frac{\omega p}{2}||\bm{N}_{i,:}^{(v)}-\frac{1}{p}\bm{K}_{i,:}^{(v)}||_2^2+\frac{\alpha}{2}\bm{N}_{i,:}^{(v)}\bm{T}_{i,:}^{(v)^T}\nonumber\\
&-\zeta_i^{(v)^T}(\bm{N}_{i,:}^{(v)}\bm{1}-1)\nonumber\\
\Leftrightarrow\min\limits_{\bm{N}_{i,:}^{(v)}\geq 0,\bm{N}_{i,i}^{(v)}=0}&\frac{\omega p}{2}||\bm{N}_{i,:}^{(v)}+\frac{\alpha}{2\omega p}\bm{T}_{i,:}^{(v)}-\frac{1}{p}\bm{K}_{i,:}^{(v)}||_2^2\nonumber\\
&-\zeta_i^{(v)^T}(\bm{N}_{i,:}^{(v)}\bm{1}-1).
\end{align}
Defining $\bm{Y}_{i,j}^{(v)}=\bm{K}_{i,j}^{(v)}/p-\alpha/(2p\omega)\bm{T}_{i,j}^{(v)}+{\zeta_i^{(v)}}/(p\omega)$, we can obtain the optimal $\bm{N}^{(v)}$ by
\begin{align}\label{nzuizhongjie}
  \bm{N}_{i,j}^{(v)} \!=\! \left\{ \begin{array}{ll}
                     \max(\bm{Y}_{i,j}^{(v)},0), & i\neq j; \\
                     0, & \textrm{otherwise,}
                   \end{array}
  \right.
\end{align}
where $\max(\bm{Y}_{i,j}^{(v)},0)$ is used to ensure that all elements of $\bm{N}^{(v)}$ are not less than 0. Based on the constraint $\bm{N}_{i,:}^{(v)}\bm{1}-1=0$, we can update $\zeta_i^{(v)}$ by
\begin{align}\label{zeta}
\zeta_i^{(v)}=\frac{\omega p}{m-1}(1-\sum_{j=1,j\neq i}^m(\frac{1}{p}\bm{K}_{i,j}^{(v)}-\frac{\alpha}{2p\omega}\bm{T}_{i,j}^{(v)})).
\end{align}
\noindent\textbf{Step 4.} Update $\bm{E}^{(v)}$. Fixing the other variables, the problem to update $\bm{E}^{(v)}$ is degraded to minimize
\begin{align}\label{je}
J(\bm{E}^{(v)})=&\frac{\omega}{2}||\bm{X}^{(v)}-\bm{X}^{(v)}\bm{Z}^{(v)}-\bm{E}^{(v)}+\frac{\bm{C}_1^{(v)}}{\omega}||_F^2+\beta||\bm{E}^{(v)}||_{\tau}.
\end{align}
Deriving $J(\bm{E}^{(v)})$ w.r.t. $\bm{E}^{(v)}$, we can have
\begin{align}\label{epian}
\frac{\partial J(\bm{E}^{(v)})}{\partial\bm{E}^{(v)}}=&\omega(\bm{E}^{(v)}+\bm{X}^{(v)}\bm{Z}^{(v)}-\bm{X}^{(v)}-\frac{\bm{C}_1^{(v)}}{\omega})+\beta\bm{D}^{(v)}_E\bm{E}^{(v)}.
\end{align}
Setting $\partial J(\bm{E}^{(v)})/\partial\bm{E}^{(v)}=0$, we can update $\bm{E}^{(v)}$ by
\begin{align}\label{ejie}
\bm{E}^{(v)}=(\bm{I}+\frac{\beta}{\omega}\bm{D}_E^{(v)})^{-1}(\bm{X}^{(v)}-\bm{X}^{(v)}\bm{Z}^{(v)}+\frac{\bm{C}_1^{(v)}}{\omega}).
\end{align}
\noindent\textbf{Step 5.} Update $\bm{F}^{(v)}$. Fixing the other variables, we can update $\bm{F}^{(v)}$ by
\begin{align}\label{jyuangai}
&\min\alpha\text{Tr}(\bm{F}^{(v)^T}\bm{L}_N^{(v)}\bm{F}^{(v)})+\gamma||\bm{F}^{(v)}\bm{F}^{(v)^T}-\bm{H}\bm{H}^T||_F^2)\nonumber\\
\Leftrightarrow&\min\alpha\text{Tr}(\bm{F}^{(v)^T}\bm{L}_N^{(v)}\bm{F}^{(v)}+\gamma(||\bm{F}^{(v)}\bm{F}^{(v)^T}||_F^2\nonumber\\
&+||\bm{H}\bm{H}^T||_F^2-\text{Tr}(\bm{F}^{(v)}\bm{F}^{(v)^T}\bm{H}\bm{H}^T)).
\end{align}
Note that $||\bm{F}^{(v)}\bm{F}^{(v)^T}||_F^2=||\bm{H}\bm{H}^T||_F^2=c$, and $c$ is a constant for a specific dataset. Thus, we transfer Eq.~\eqref{jyuangai} into
\begin{align}\label{jf}
&\min\text{Tr}(\bm{F}^{(v)^T}(\alpha\bm{L}_N^{(v)}-\gamma\bm{H}\bm{H}^T)\bm{F}^{(v)}).
\end{align}
Obviously, we can solve Eq.~\eqref{jf} by the eigenvalue decomposition. The optimal $\bm{F}^{(v)}$ is the eigenvector set that corresponds to the first $c$ smallest eigenvalues of matrix $(\alpha\bm{L}_N^{(v)}-\gamma\bm{H}\bm{H}^T)$.

\noindent\textbf{Step 6.} Update $\bm{H}$. Similar to Step 5,
we update $\bm{H}$ by
\begin{align}\label{jfxing}
&\min-\gamma\sum_v\text{Tr}(\bm{F}^{(v)}\bm{F}^{(v)^T}\bm{H}\bm{H}^T)\nonumber\\
\Leftrightarrow&\max\text{Tr}(\bm{H}^T(\sum_v(\gamma\bm{F}^{(v)}\bm{F}^{(v)^T}))\bm{H}).
\end{align}
By the eigenvalue decomposition, we can learn the optimal $\bm{H}$, which is also the eigenvector set corresponding to the first $c$ largest eigenvalues of matrix $(\sum_v(\gamma\bm{F}^{(v)}\bm{F}^{(v)^T}))$.

\noindent\textbf{Step 7.} Update $\bm{C}_1^{(v)},\bm{C}_2^{(v)}$ and $\omega$. We can update them by
\begin{align}\label{jc}
\bm{C}_1^{(v)}&=\bm{C}_1^{(v)}+\omega(\bm{X}^{(v)}-\bm{X}^{(v)}\bm{Z}^{(v)}-\bm{E}^{(v)})\nonumber\\
\bm{C}_2^{(v)}&=\bm{C}_2^{(v)}+\omega(\bm{W}^{(v)^T}\bm{Z}\bm{W}^{(v)}-\bm{M}+p\bm{N}^{(v)})\nonumber\\
\omega&=\min(\varphi\omega,\omega_{max}),
\end{align}
where $\varphi$ and $\omega_{max}$ are constants.

The V$^3$H algorithm is shown in Algorithm~\ref{simc_alg}. We provide its codes in Code Ocean (DOI:10.24433/CO.2119636.v1) and Github (\url{https://github.com/ZeusDavide/TAI_V3H.git}).
\subsection{Convergence and Complexity}\label{subsection:fuzadu}
\subsubsection{Convergence Analysis}\label{subsubsection:shoulianfenxi}
To optimize our proposed V$^3$H, we need to solve seven subproblems in Algorithm \ref{simc_alg}. Each subproblem has a closed solution w.r.t the corresponding variable.
The objective function is bounded, and all the above seven steps do not increase the objective function value. Thus, the objective function can reduce monotonically to a stationary value, and V$^3$H can at least find a locally optimal solution.
\subsubsection{Complexity Analysis}\label{subsubsection:fuzadufenxi}
From Section \ref{subsection:optimization}, the major computational costs of our proposed V$^3$H mainly come from the operations like matrix inverse, SVD, and eigenvalue decomposition. Therefore,
Steps 1, 5, and 6 are the main computational costs. For Step 2, its major computational costs are from the inverse operation $(\bm{I}+\bm{X}^{(v)^T}\bm{X}^{(v)})^{-1}$. Since both $\bm{I}$ and $\bm{X}^{(v)}$ are not updated in each iteration, we can pre-compute the inverse operation before the iteration for simplicity.
For Steps 1, 5, and 6, they have the same computational complexity $O(n^3)$.
Therefore, the whole computational complexity of V$^3$H is about $O(in_vn^3)$, where $i$ is the iteration number, $n_v$ is the view number, and $n$ is the sample number.

Note that the complexity of V$^3$H has nothing to do with the feature dimension $d_v$. Since most real-world data are high-dimensional \cite{6809191}, V$^3$H will have wide applications.

\section{Performance Evaluation}\label{section:exp}
We first illustrate the clustering performance of the proposed V$^3$H, then verify V$^3$H's convergence, and finally analyze the sensitivity of V$^3$H's parameters.
\subsection{Datasets} \label{section:data}
\begin{table}[t]
\centering
\caption{Statistics of the datasets, where ``\#. features'' means the total feature dimension of all the views.}
\begin{tabular}{lccccccccccccccccccccc}
\hline
Dataset  & 3-Sources & 20-NGs  & 100-Ls & BBC (3v)& BBC (4v) \\\hline
\#. samples & 169  & 500    & 1600 & 282 &685 \\
\#. views    &3  & 3    & 3  &3 &4\\
\#. features    &10259  & 1500    & 192  &7591 &18641\\
\#. clusters & 6  & 5  & 100&5  &5 \\
\hline
Dataset &BS (2v) &BS (4v)& BUAA & Coil& Digit \\\hline
\#. samples &544&116 &180&1440 & 2000  \\
\#. views  &2&4&2 &3   &5\\
\#. features    &6386&8345&200&11078   &585\\
\#. clusters &5&5&3 &20& 10 \\
\hline
Dataset   & NUS & ORL &Scene & Yale &YaleB \\\hline
\#. samples & 2400    & 400 & 2688&165&650 \\
\#. views  & 6    & 3  &4&3&3\\
\#. features     &1134    & 14150 &1248&14150&12554\\
\#. clusters & 12  & 40&8&15&10\\
\hline
\end{tabular}
\label{dataset}
\end{table}
\begin{figure*}[t]
\centering
\includegraphics[width=1\textwidth]{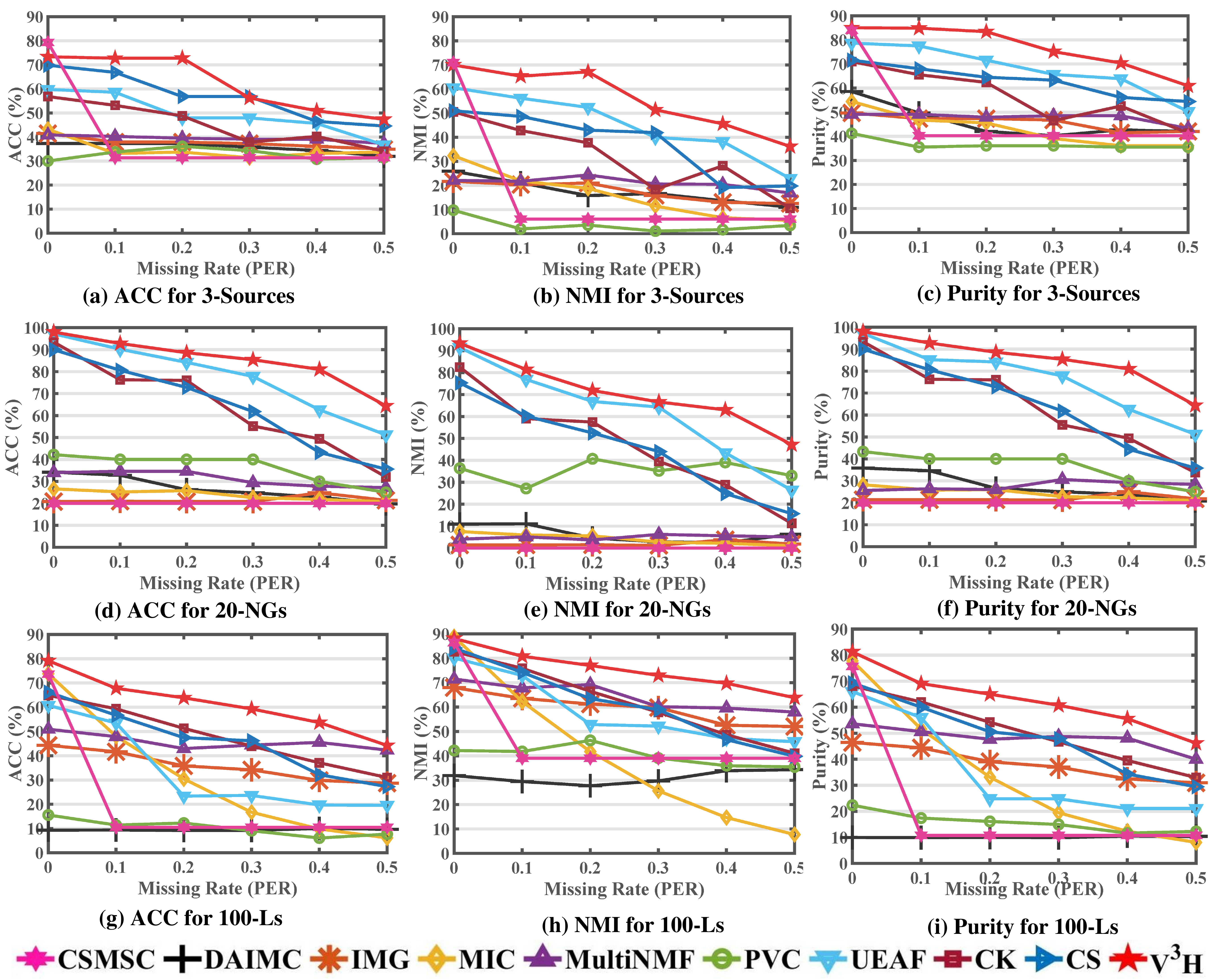}
\vspace{-0.75cm}
\caption{Results on 3-Sources, 20-NGs, and 100-Ls datasets.}
\label{fig:diyibufen}
\end{figure*}
\begin{figure*}[t]
\centering
\includegraphics[width=1\textwidth]{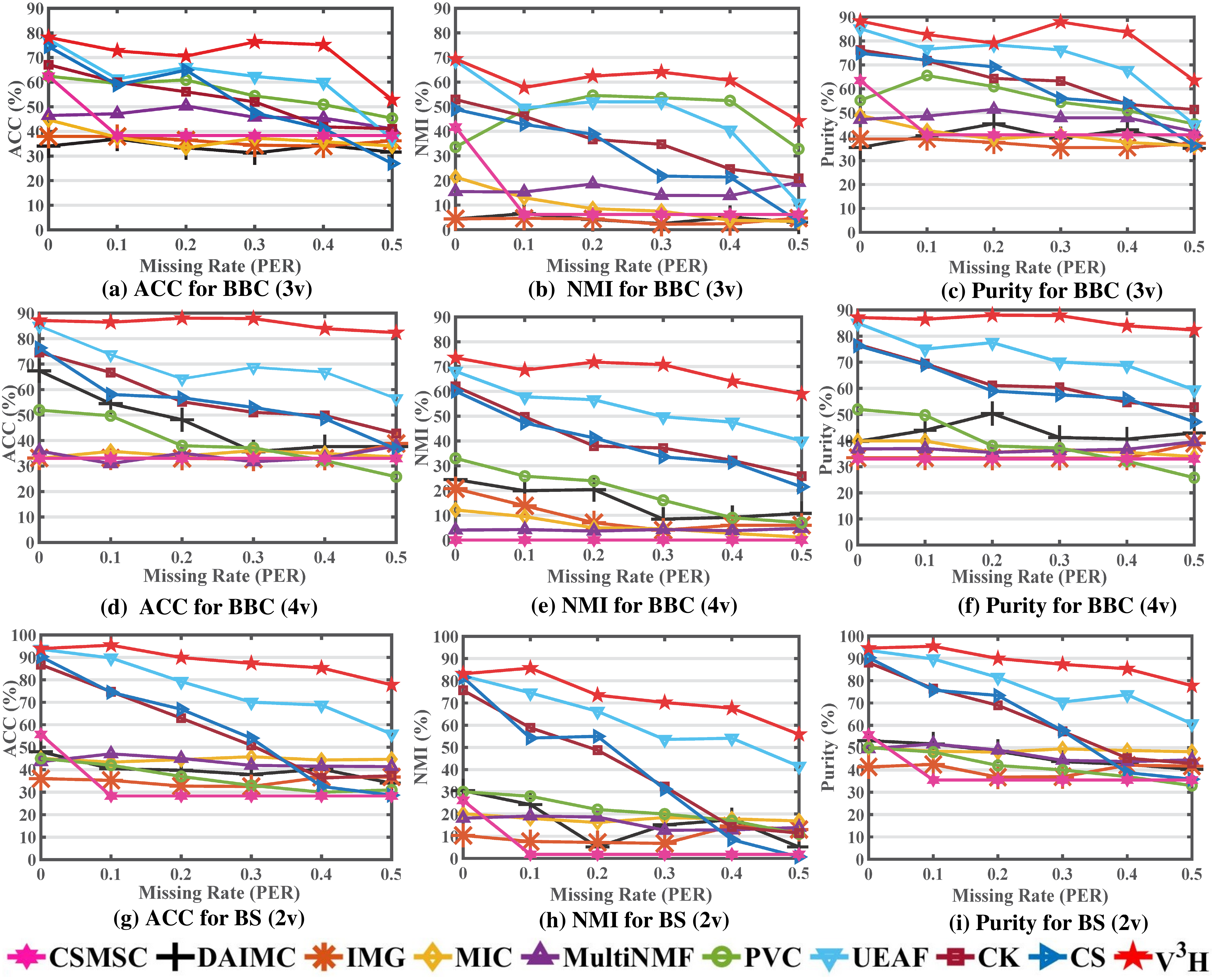}
\vspace{-0.75cm}
\caption{Results on BBC (3v), BBC (4v) and BS (2v) datasets.}
\label{fig:dierbufen}
\end{figure*}
\begin{figure*}[t]
\centering
\includegraphics[width=1\textwidth]{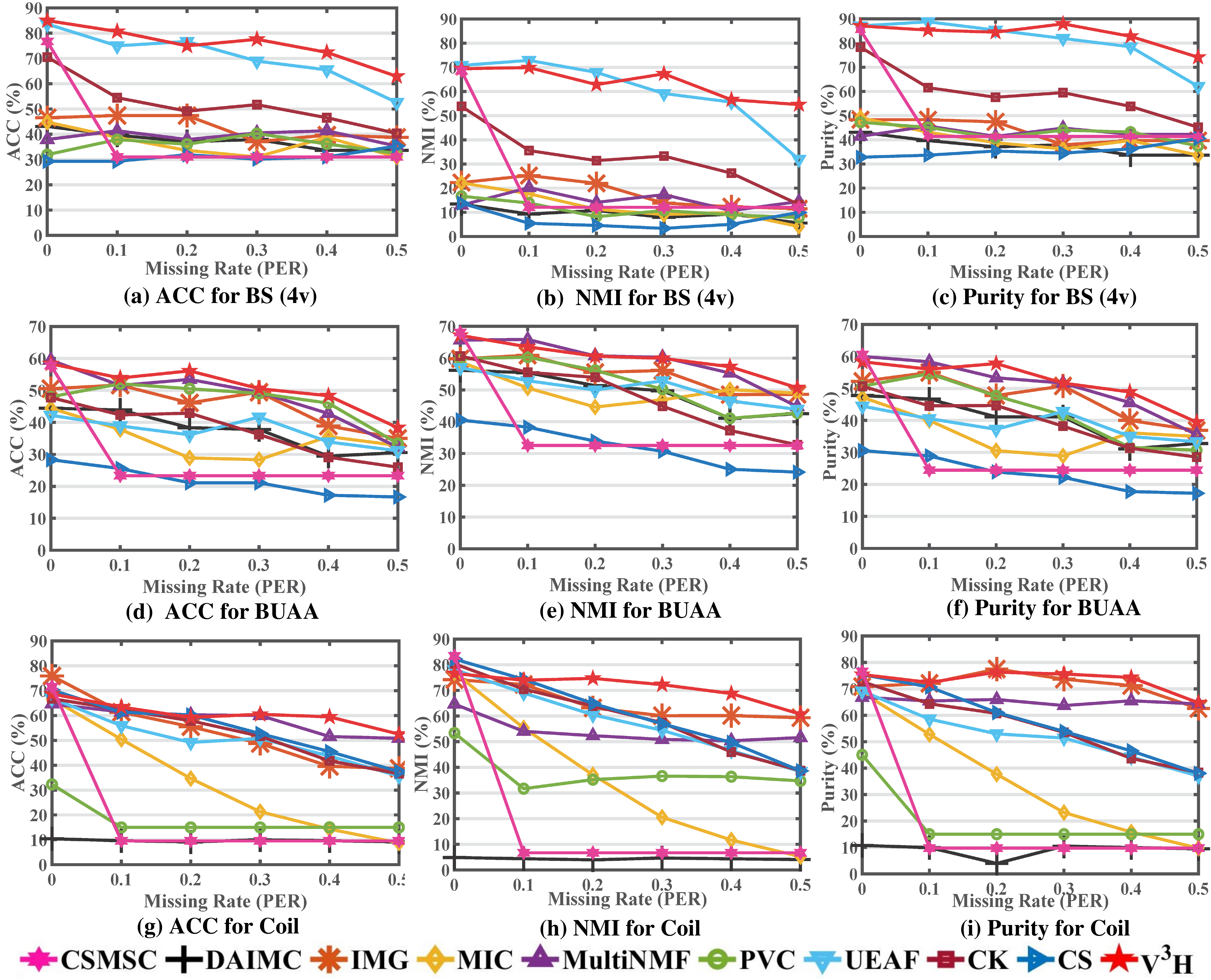}
\vspace{-0.75cm}
\caption{Results on BS (4v), BUAA, and Coil datasets.}
\label{fig:disanbufen}
\end{figure*}
\begin{figure*}[t]
\centering
\includegraphics[width=1\textwidth]{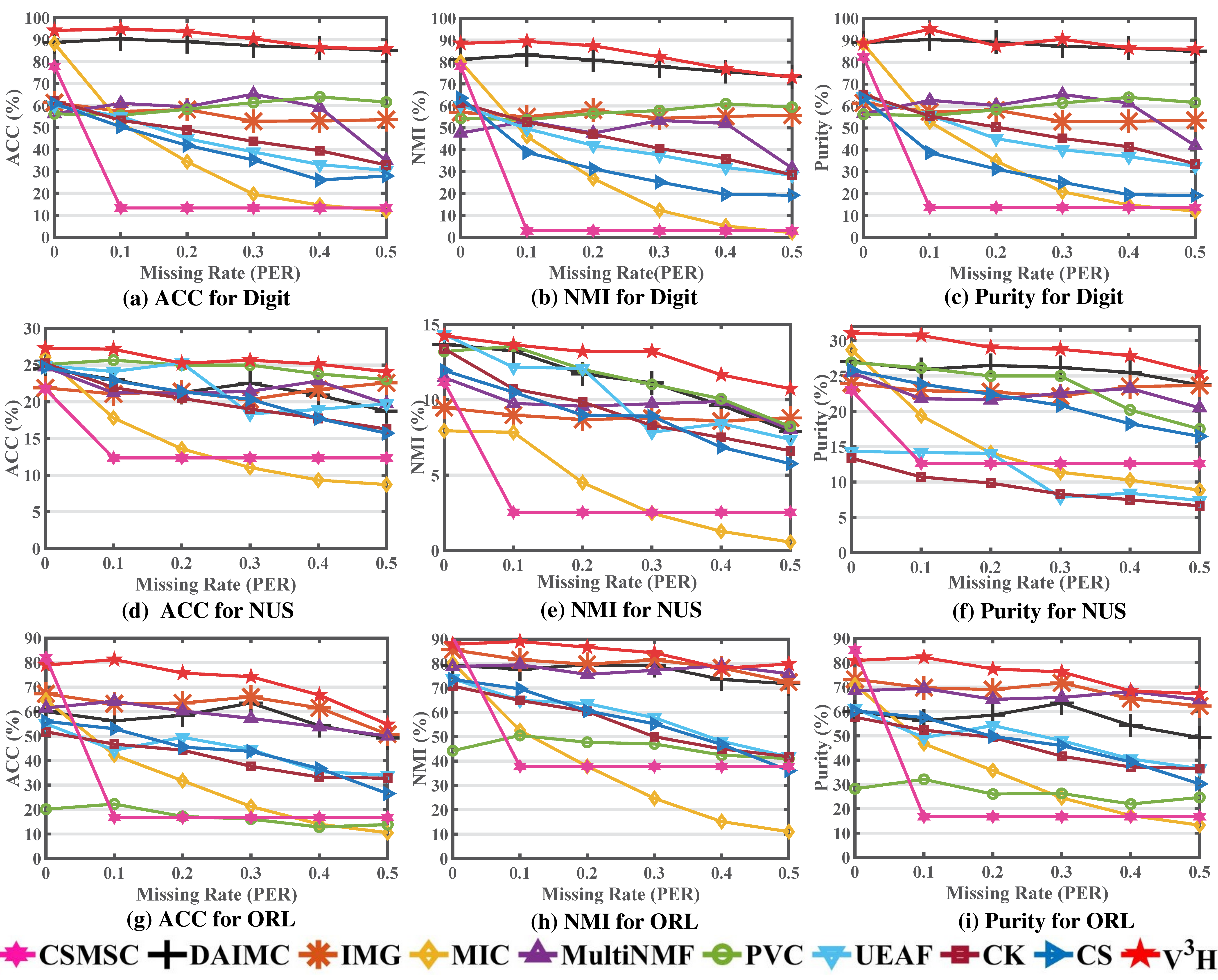}
\vspace{-0.75cm}
\caption{Results on Digit, NUS, and ORL datasets.}
\label{fig:disibufen}
\end{figure*}
\begin{figure*}[t]
\centering
\includegraphics[width=1\textwidth]{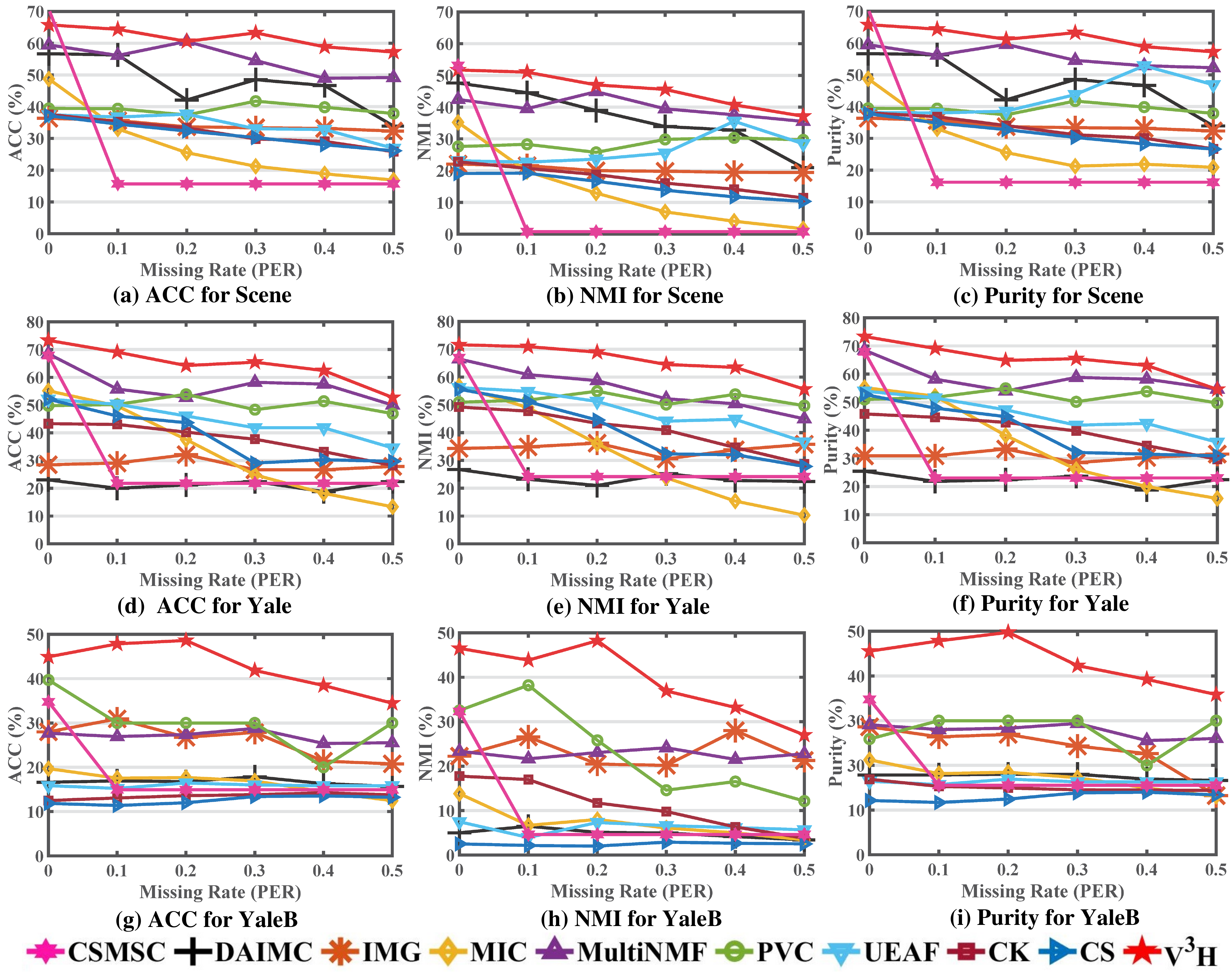}
\vspace{-0.75cm}
\caption{Results on Scene, Yale, and YaleB datasets.}
\label{fig:diwubufen}
\end{figure*}
\begin{figure*}[t]
\centering
\subfigure[Study for $\alpha$ and $\beta$]{\label{fig:NMI_3D} \includegraphics[width=0.315\textwidth]{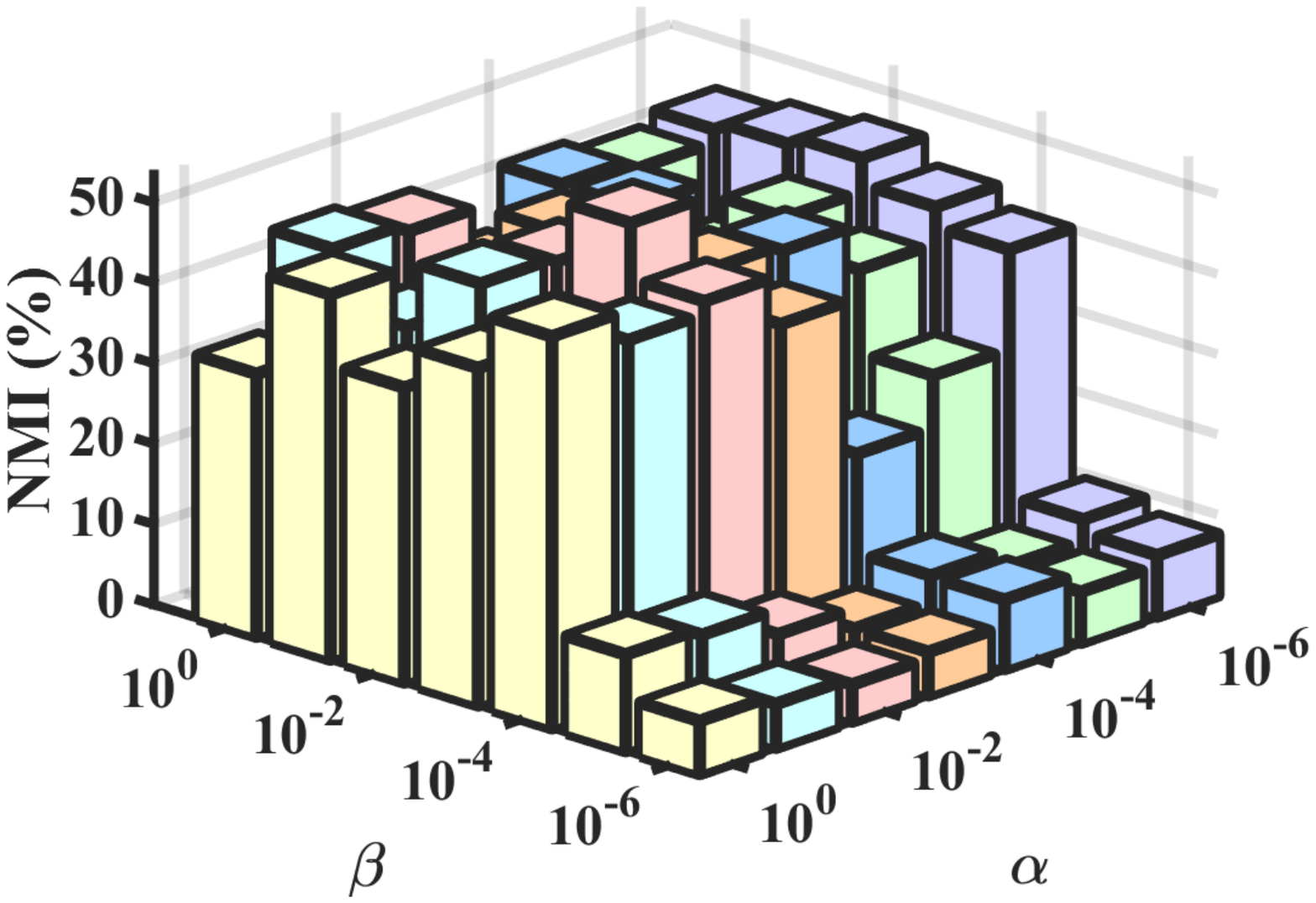} }
\subfigure[Study for $\gamma$]{\label{fig:NMI_2D} \includegraphics[width=0.315\textwidth]{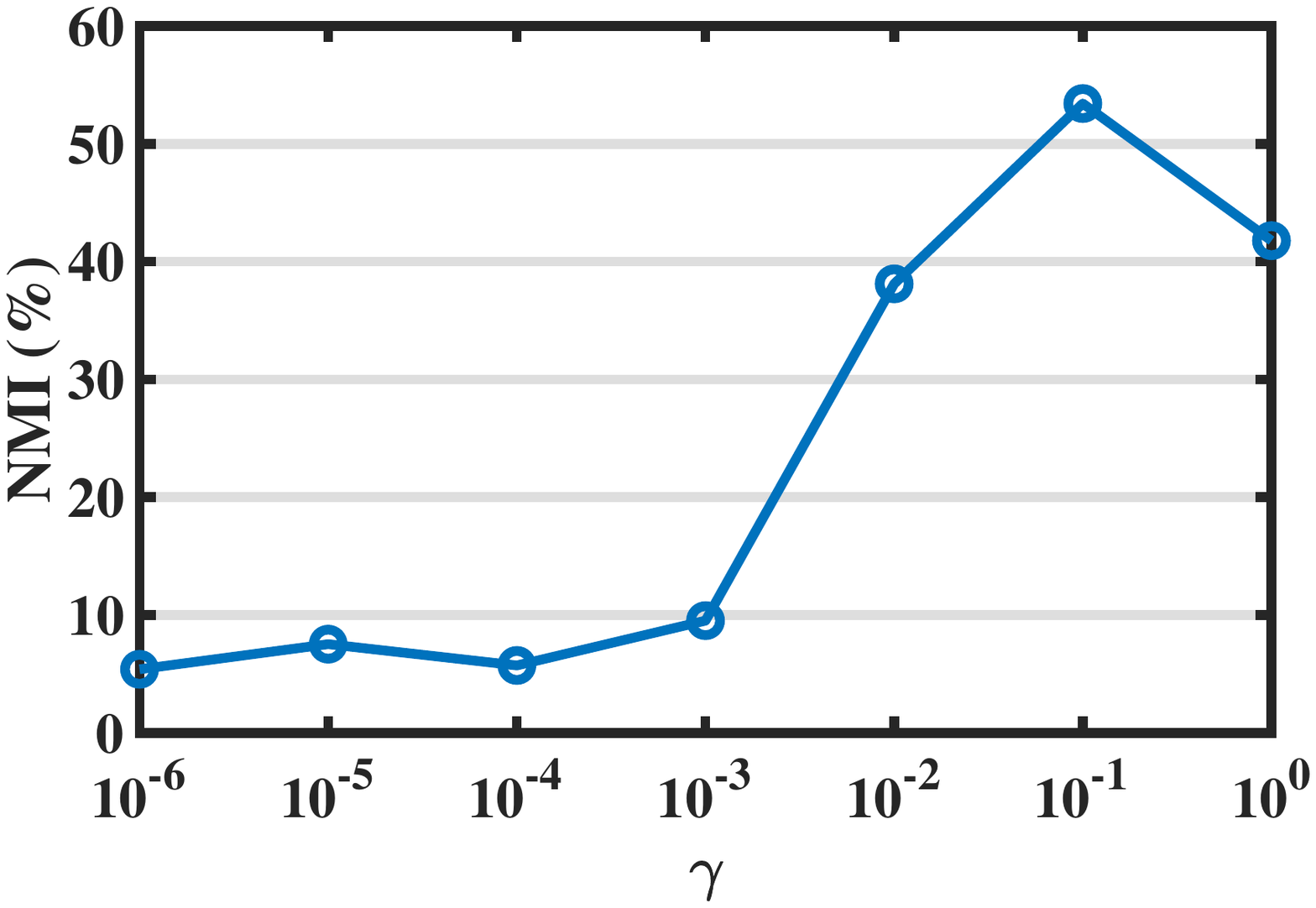} }
\subfigure[Convergence study on Coil]{\label{fig:huigui_COIL20} \includegraphics[width=0.315\textwidth]{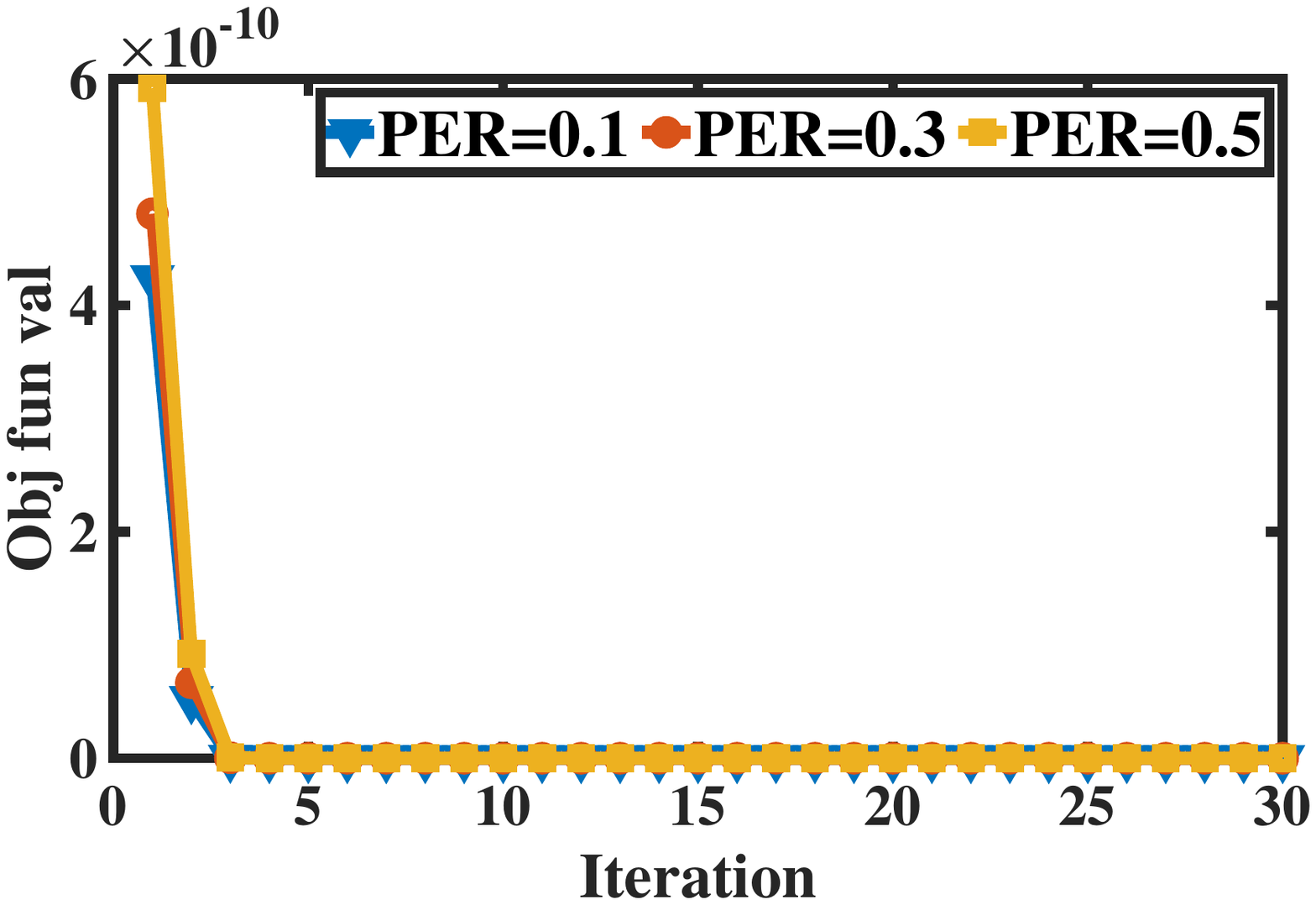} }
\caption{Convergence and parameter studies.}
\label{fig:chaocanshu}
\end{figure*}
We conduct experiments on fifteen well-known popular datasets: 3-Sources\footnote{\url{http://mlg.ucd.ie/datasets/3sources.html}.}, 20 New Groups (20-NGs)\footnote{\url{http://lig-membres.imag.fr/grimal/data.html}.}, 100 Leaves (100-Ls)\footnote{\url{https://archive.ics.uci.edu/ml/datasets/One-hundred+plant+species+leaves+data+set}.}, BBC with 3 views (BBC (3v))\footnote{\url{http://mlg.ucd.ie/datasets/segment.html}.}, BBC with 4 views (BBC (4v))\footnote{\url{http://mlg.ucd.ie/datasets/segment.html}.}, BBCSport with 2 views (BS (2v))\footnote{\url{http://mlg.ucd.ie/datasets/segment.html}.}, BBCSport with 4 views (BS (4v))\footnote{\url{http://mlg.ucd.ie/datasets/segment.html}.}, BUAA \cite{huang2012buaa}, Coil~\cite{NENE1996Columbia}, Digit\footnote{\url{http://archive.ics.uci.edu/ml/datasets.html}.}, NUS~\cite{Chua2009NUS}, ORL~\cite{Harter1994Harter}, Outdoor Scene (Scene)~\cite{article}, Yale\footnote{\url{http://vision.ucsd.edu/content/yale-face-database}.}, and Extended YaleB (YaleB)~\cite{KriegmanFrom}.

For these datasets, most multi-view clustering algorithms often cluster their common subsets for simplicity. To compare fairly with these algorithms, we use the same size datasets for clustering.
The important statistics of used datasets are shown in Table~\ref{dataset}, and the detailed statistics are as follows:

1) \textbf{3-Sources}: it is a news dataset that has 948 samples collected from 3 views: BBC with 3560 features, Reuters with 3631 features and The Guardian with 3068 features. Following~\cite{wen2019unified}, we select a subset with 169 samples, which are categorized into  6 clusters.

2) \textbf{20-NGs}: it is a document dataset that has 500 samples from 3 views. Each view has 500 features.
These documents are categorized into 5 clusters.

3) \textbf{100-Ls}: it has 1600 samples from 100 clusters. Each sample appears in 3 views, and each view has 64 features.

4) \textbf{BBC (3v)} and 5) \textbf{BBC (4v)}: the original BBC dataset has  685 samples, which are described by 3-4 views categorized into 5 clusters.
Following~\cite{WengZCPC20}, we choose a subset with 282 samples described by 3 views. These views include 2582 features, 2544 features, 2465 features, respectively.
Following~\cite{8662703}, we also choose the full dataset described by 4 views. These views include 4659 features, 4633 features, 4665 features, and 4684 features, respectively.

6) \textbf{BS (2v)} and 7) \textbf{BS (4v)}: the original BBCSport dataset has 737 samples, which are described by 2-4 views and categorized into 5 clusters.
Following~\cite{ZhouLLZLZY20}, we select a subset with 544 samples described by 2 views.
These views include 3183 features and 3203 features, respectively.
Following~\cite{wen2019unified}, we also use a subset with 116 samples described by 4 views. These views include 1991 features, 2063 features, 2113 features, and 2158 features, respectively.

8) \textbf{BUAA}: it has 180 image samples, which are categorized into 10 clusters. Each image is described by 2 views, and each view has 100 features.

9) \textbf{Coil}: it is an image dataset that has 1440 samples consisting of 20 clusters. Each image appears in three different views: Intensity with 1024 features, LBP feature with 3304 features, and Gabor feature with 6750 features.

10) \textbf{Digit}: it has 2000 samples categorized into 10 clusters.
It has 5 views: FOU with 76 features, FAC with 216 features, KAR with 64 features, PIX with 240 features, and ZER with 47 features.

11) \textbf{NUS}: it consists of 2400 samples categorized into 12 clusters.
Besides, each image is described by 6 views: CH with 65 features, CM with 226  features, CORR with 145 features, ED with 74 features and WV with 129 features.

12) \textbf{ORL}: it is an image dataset, which is described by 3 views. The dataset has 400 samples categorized into 40 clusters. These views include 4096 features, 3304 features and 6750 features, respectively.

13) \textbf{Scene}: it is an image dataset that  has  2,688 samples consisting of  8 clusters. Each image is described by 4 views: GIST with 512 features, color moment with 432 features, HOG with 256 features and LBP with 48 features.

14) \textbf{Yale}: it has 165 samples categorized into 15 clusters. It is described by 3 views: Intensity with 4096 features, LBP with 3304 features and Gabor with 6750 features.


15) \textbf{YaleB}: it is a face image dataset that has 2414 samples described by 3 views.
These views include 2500 features, 3304 features and 6750 features, respectively.
Following~\cite{WangGLZL17}, we choose a subset with 650 samples.
\subsection{Compared Methods} \label{section:compared}
We compare our proposed \textbf{V$^3$H} with the following state-of-the-art methods, which are most relevant to our work:

1) \textbf{CSMSC}~\cite{luo2018consistent}
learns a consistent representation and a set of specific representations from complete multi-view data;

2) \textbf{DAIMC}~\cite{hu2018doubly} extends MIC based on weighted semi-NMF and $L_{2,1}$-Norm regularization regression;

3) \textbf{IMG}~\cite{zhao2016incomplete} transforms the collected incomplete multi-view data to a complete representation in a latent space;

4) \textbf{MIC}~\cite{shao2015multiple} extends MultiNMF based on weighted NMF and $L_{2,1}$-Norm regularization;

5) \textbf{MultiNMF}~\cite{liu2013multi} extends NMF to multi-view scenes by jointing these views;

6) \textbf{PVC}~\cite{li2014partial} aligns the same samples in  different views by constructing a latent subspace;

7) \textbf{UEAF}~\cite{wen2019unified} learns a consensus representation for all views by extending MIC;

8) \textbf{Concat-K-means clustering (CK)}, and 9) \textbf{Concat-Spectral clustering (CS)}. CK and CS are two baselines that concatenate all views into one single view for clustering. (i) CK: we first fill the missing samples with the average features for each view. Then we concatenate the features of all the views, and perform K-means clustering on the concatenated view.
(ii) CS: similar to CK, we perform spectral clustering on the concatenated view.

For CSMSC and MultiNMF, they cannot directly handle incomplete multi-view data. Following~\cite{hu2018doubly}, we first fill the missing samples with average feature values. Then we perform CSMSC and MultiNMF. Since IMG and PVC cannot perform clustering on the incomplete data with more than two views, we use these methods on all the two-views combinations and report average results for fairness. Since our proposed V$^3$H has three parameters, $\alpha$, $\beta$, and $\gamma$, we adjust them to get the best performance (see Section~\ref{subsection:chaocanshu}). Since $p$, $\eta$, $\tau$, and $w_i$ are adjustable as needed, we set $p=1$, $\eta=10^{-3}$, $\tau=10^{-2}$, and $w_i=1$ in our experiment for simplicity.

Following \cite{zhao2016incomplete}, we repeat each incomplete multi-view clustering experiment $10$ times to obtain the average performance. Following~\cite{hu2018doubly} and \cite{shao2015multiple}, we randomly delete some samples from each view to get incomplete views. We set the missing rate (PER) from $0$ (each view is complete) to $0.5$ (each view has $50\%$ samples missing) with $0.1$ as an interval.

\textbf{Evaluation Metric}:
Following~\cite{wen2019unified}, we evaluate the experimental results by three popular metrics: Accuracy (ACC), Normalized Mutual Information (NMI) and Purity. For these metrics, the larger value represents better performance.
\subsection{Clustering Performance and Analysis}  \label{subsection:shiyanjieguo}
Fig.~\ref{fig:diyibufen}-\ref{fig:diwubufen} show the clustering results on these real-world datasets. Obviously, our proposed V$^3$H significantly performs better than other state-of-the-art methods in most cases. Especially, when we cluster the YaleB dataset with PER=0.2  (Fig.
7(g), 7(h), and 7(i)), V$^3$H gains large improvements at least $ 23.24\%$ in ACC, $ 22.42\%$ in NMI, and $22.46\%$ in Purity over the best performing compared method PVC.

For convenience, we first divide the compared methods into 3 groups: single-view methods (CK and CS), two-view methods (IMG and PVC) and multi-view methods (CSMSC, DAIMC, MIC, MultiNMF and UEAF). Then, based on the experimental results, we compare and analyze the performance of different groups. Finally, we critically analyze each method.

\textbf{V$^3$H versus single-view methods}: compared with single-view methods, V$^3$H achieves better performance on all the datasets. For instance, when clustering the Scene dataset with PER=0.5 (Fig. 7(a), 7(b), and 7(c)), compared with CK and CS, V$^3$H raises the clustering results at least $31.10\%$ in ACC, $25.73\%$ in NMI, and $30.54\%$ in Purity, respectively. It is because the Scene dataset has 4 views.
CK and CS only simply concatenate these views, which cannot learn the relationship information between different views.
On the contrary, V$^3$H can extract the relationship information by integrating different views.
Therefore, integrating effectively different views is necessary for multi-view clustering.

\textbf{V$^3$H versus two-view methods}: compared with two-view methods, V$^3$H obtains better clustering results in all the cases. For the datasets with more than 2 views (e.g., BS (4v), 3-Sources, 20-NGs, etc.), V$^3$H performs better than PVC and IMG.
When we cluster the BS (4v) dataset with PER=0.5 (Fig. 5(a), 5(b) and 5(c)), compared with PVC and IMG, V$^3$H raises the clustering results at least $24.14\%$ in ACC, $43.10\%$ in NMI, and $34.47\%$ in Purity, respectively.
The reason is that PVC and IMG can only integrate two views, and structural information of the rest views is not learned, which illustrates the necessity of integrating all the views.
As for the two-view datasets (e.g., BUAA and BS (2v)), V$^3$H still outperforms PVC and IMG.
For the BUAA dataset with PER=0.2  (Fig. 5(d), 5(e) and 5(f)), compared with PVC and IMG, V$^3$H improves the performance at least $3.51\%$ in ACC, $1.07\%$ in NMI, and $5.97\%$ in Purity, respectively. The main reason is that the BUAA dataset contains the nonlinear structural information, V$^3$H can learn the information based on the linear kernel used in the Laplacian for spectral clustering.

\textbf{V$^3$H versus multi-view methods}: compared with multi-view methods, V$^3$H also achieves better clustering performance in most cases.  When we cluster the BBC (4v) dataset with PER=0.5 (Fig. Fig. 4(d), 4(e), and 4(f)), compared with multi-view methods, V$^3$H raises the performance at least $25.84\%$ in ACC, $18.95\%$ in NMI, and $22.92\%$ in Purity, respectively.
It is because each view of the BBC (4v) dataset includes the unique information.
V$^3$H can learn the unique information through the corresponding variation matrix, while the unique information is ignored by multi-view methods. More impressively, as PER on the dataset increases,
V$^3$H achieves satisfactory and relatively stable clustering results, while the performance of all multi-view methods drops significantly. It is because based on the subspace decomposition in Eq.~\eqref{Zfenjie}, V$^3$H can learn the information from both the presented samples and the missing samples. But these multi-view methods can only learn information from the presented samples.

In summary, all the methods are analyzed as follows:

1) \textbf{CK and CS}: for most clustering tasks, the performance of CK and CS are close because they simply concatenate all views. Although this concatenation is easy to operate, CK and CS always perform poorly due to ignoring the relationship between different views. Therefore, CK and CS often have limited applications in multi-view datasets.

2) \textbf{PVC and IMG}: for the incomplete two-view datasets, PVC can obtain pretty clustering performance by establishing a latent subspace from two views. Similarly, IMG also performs well in incomplete two-view clustering by introducing manifold learning into PVC.
But when clustering the data with more than $2$ incomplete views, PVC and IMG cannot obtain an optimal latent subspace due to ignoring the global structure of the multi-view data. Thus, it is difficult for them to obtain satisfactory results in incomplete multi-view clustering tasks.

3) \textbf{CSMSC, DAIMC, MIC, MultiNMF, and UEAF}: when clustering complete multi-view datasets, CSMSC, MIC, and MultiNMF can perform well by learning a subspace from each view. As the missing rate increases, the clustering results of these three methods drop significantly.
The reason is that these three methods simply fill the missing samples with the average feature values, which neglects the hidden information of the missing samples.
Since real-world data are often incomplete, these three methods are difficult to be widely used.
For these datasets with little alignment information, DAIMC and UEAF always obtain unsatisfactory clustering results because they learn the consensus representation by aligning these views. These drawbacks limit the application of these methods.

4) \textbf{Our proposed V$^3$H}: by aligning cluster indicator matrices from different views and learning the low-rank representation, V$^3$H can perform satisfactorily in most cases, which shows its wide application.
Moreover, when the missing rate is relatively large (e.g., PER=0.4 or PER=0.5), V$^3$H has more obvious superiority over other state-of-the-arts, which illustrates its effectiveness in high-incompleteness applications.
\subsection{Parameter Sensitivity}\label{subsection:chaocanshu}
In terms of $\{\alpha, \beta, \gamma\}$, we conduct the hyper-parameter experiments on 3-Sources dataset. Similar to~\cite{hu2018doubly}, we set PER=0.5 and report V$^3$H's NMI versus $\alpha$, $\beta$, and $\gamma$ within the set of $\{10^0,10^{-1},10^{-2},10^{-3}, 10^{-4},10^{-5},10^{-6}\}$.

As shown in Fig.~\ref{fig:NMI_3D} and \ref{fig:NMI_2D}, our proposed V$^3$H obtains stable and satisfactory clustering performance across a wide range of these parameters. Thus, V$^3$H is insensitive to the variation of the parameters.
Also, V$^3$H obtains the best clustering results when we set $\alpha$=$10^{-3}$, $\beta$= $10^{-4}$ and $\gamma$=$10^{-1}$, which are the recommended values.
\subsection{Convergence Study}\label{subsection:shoulian}
Based on the recommended values of these hyper-parameters, we study the convergence by conducting the experiments on the Coil dataset with different PERs, i.e., PER=0.1, PER=0.3, PER=0.5.
Fig.~\ref{fig:huigui_COIL20} show that the convergence curve versus the iteration number, and ``Obj fun val" represents ``objective function value", which is calculated by $(||\bm{M}||_{\eta}+\sum_v(\beta||\bm{E}^{(v)}||_{\tau}+\alpha\text{Tr}(\bm{F}^{(v)^T}\bm{L}_N^{(v)}\bm{F}^{(v)})-\gamma \text{Tr}(\bm{F}^{(v)}\bm{F}^{(v)^T}\bm{F}^{\ast}\bm{F}^{\ast^T})))$ $/(||\bm{X}^{(v)}-\bm{X}^{(v)}\bm{Z}^{(v)}-\bm{E}^{(v)}||_F^2+||\bm{W}^{(v)^T}\bm{Z}^{(v)}\bm{W}^{(v)}-\bm{M}+p\bm{N}^{(v)}||_F^2)$, similar to~\cite{wen2018incompleteb}. Obviously, our proposed V$^3$H has converged just after 10 iterations for all PERs, which shows its fast convergence.


Note that the convergence curves under different PERs are close to each other. This is because, based on Eq.~\eqref{Zfenjie}, we can learn the information of all samples (including presented samples and missing samples). When PER changes, the dimensions of $\bm{M}$ and $\bm{N}^{(v)}$ do not change.

Besides, when the iteration number is the same, $obj(PER=0.5)>obj(PER=0.3)>obj(PER=0.1)$. The reason is as follows: in the objective function, only the dimensions of $\bm{E}$ and $\bm{Z}^{(v)}$ will change as the missing rate increases. Thus, only the values of $||\bm{E}^{(v)}||_{\tau}$ and $||\bm{X}^{(v)}-\bm{X}^{(v)}\bm{Z}^{(v)}-\bm{E}^{(v)}||_F^2$ will rely on missing rate. In fact, for a robust algorithm, its error matrix $\bm{E}$ will generally be small. Therefore, we can approximate the numerator of $obj$ as a constant. As the missing rate increases, the value of $||\bm{X}^{(v)}-\bm{X}^{(v)}\bm{Z}^{(v)}-\bm{E}^{(v)}||_F^2$ will decrease, and the objective function value will increase.

\section{Conclusion and Future Work}  \label{section:con}
In this paper, we propose a novel \textbf{V}iew \textbf{V}ariation and \textbf{V}iew \textbf{H}eredity approach (V$^3$H) for incomplete multi-view clustering. As far as we know, V$^3$H is the \textbf{first} attempt to introduce genetics into the clustering method. Also, it can learn the consistent information and the unique information based on view variation and view heredity respectively.
Extensive experiments on fifteen datasets demonstrate the superiority of V$^3$H over other state-of-the-art methods. Impressively, when clustering the YaleB dataset with the missing rate of 0.2, V$^3$H
improves
at least $ 23.24\%$ in ACC, $22.42\%$ in NMI, and $22.46\%$ in Purity over the best performing compared method.

Our proposed V$^3$H is an offline approach for high-dimensional incomplete multi-view clustering.
A larger challenge is to cluster large-scale high-dimensional data.
In the future, we will introduce online learning into V$^3$H for the large-scale high-dimensional data about COVID-19. We collect a large amount of data about COVID-19 every day, and online learning is an effective way to process these data.
Based on online learning, we will process these data.

%
%
%
%
%

\ifCLASSOPTIONcaptionsoff
  \newpage
\fi
\bibliographystyle{IEEEtran}
\bibliography{IEEEtran}
\end{document}